\documentclass[11pt]{article}
\usepackage{fullpage}
\usepackage{microtype}
\usepackage{graphicx}
\usepackage{subfigure}
\usepackage{booktabs} % for professional tables

\usepackage{hyperref}

\usepackage{amsmath}
\usepackage{amssymb}
\usepackage{mathtools}
\usepackage{amsthm}
\usepackage{algorithm}
\usepackage{algorithmic}
\usepackage[capitalize,noabbrev]{cleveref}
\usepackage{natbib}
\theoremstyle{plain}
\newtheorem{theorem}{Theorem}[section]

\newtheorem{lemma}[theorem]{Lemma}
\newtheorem{corollary}[theorem]{Corollary}
\theoremstyle{definition}
\newtheorem{definition}[theorem]{Definition}
\newtheorem{assumption}[theorem]{Assumption}
\theoremstyle{remark}
\newtheorem{remark}[theorem]{Remark}

\usepackage[textsize=tiny]{todonotes}

\newcommand{\T}{\mathcal{T}}
\renewcommand{\S}{\mathcal{S}}
\renewcommand{\P}{\mathbb{P}}
\newcommand{\E}{\mathbb{E}}

\title{Enhancing selectivity using Wasserstein distance based reweighing}
\author{Pratik Worah\thanks{Google inc., pworah@google.com}}
\begin{document}
\maketitle
\begin{abstract}
Given two labeled data-sets $\mathcal{S}$ and $\mathcal{T}$, we design a simple and efficient greedy algorithm to reweigh the loss function such that the limiting distribution of the neural network weights that result from training on $\mathcal{S}$ approaches the limiting distribution that would have resulted by training on $\mathcal{T}$.  

On the theoretical side, we prove that when the metric entropy of the input datasets is bounded, our greedy algorithm outputs a close to optimal reweighing, i.e., the two invariant distributions of network weights will be provably close in total variation distance. Moreover, the algorithm is simple and scalable, and we prove bounds on the efficiency of the algorithm as well.

As a motivating application, we train a neural net to recognize small molecule binders to MNK2 (a MAP Kinase, responsible for cell signaling) which are non-binders to MNK1 (a highly similar protein).
In our example dataset, of the 43 distinct small molecules predicted to be most selective from the enamine catalog, 2 small molecules were experimentally verified to be selective, i.e., they reduced the enzyme activity of MNK2 below 50\% but not MNK1, at 10$\mu$M -- a 5\% success rate.
\end{abstract}

\section{Introduction}

Deep learning has found applications in diverse areas ranging from organic chemistry to computer generated art. Its applicability is limited by the availability of large amounts of labeled training data. 
%A typical all-to-all neural net of width $n$ (say $n=10000$) and depth $d$ (say $d=10$) has $\Theta(n^2d)$ ($\simeq 10^{8}$) weight parameters, and although many neural nets work well despite being somewhat over-parameterized, they still need the number of training examples to be of a similar or only slightly smaller order of magnitude.
% Such large, high dimensional, training data-sets present several difficulties. For example, if the test data was collected at a different time, or via a different method than the training data, then we will have a larger than expected generalization error. This is known as distribution shift in machine learning literature (see for example~\cite{dataset-book}).
This reliance on large amounts of training data can lead to difficulties in training. For example, suppose we separately gather training datasets $\S$ and $\T$, that consist of color images of cats and dogs, for a neural net to perform two different classification tasks: (1) classify cats vs dogs and (2) classify red vs blue. Now, if we want a new classifier that uses $\S$ to more accurately distinguish red cats vs other colored cats, then the usual approach would be to train another neural net on labeled samples of red cats vs others. If the category red cats represents disproportionately few labeled samples in $\S$ compared to $\T$, then we are potentially out of luck, unless we gather more labeled data. In general, this may be unavoidable, but when the underlying ground space for the two sampled datasets is the same, then we may be able to use information from $\T$ in $\S$, especially if nearby objects are likely to be mapped to the same class. In this paper, we design a scalable algorithm (see Algorithm~\ref{algmain}) that reweighs points of $\mathcal{S}$ (using $\mathcal{T}$), so that the limiting distribution of post-training network weights using reweighted $\mathcal{S}$ is provably close to what would be obtained by training on $\mathcal{T}$;  or more generally, some user specified weighted mixture of $\S$ and $\T$ (see theorems in Section~\ref{theory:sec}). We also demonstrate the effectiveness of our reweighting algorithm in a practical application: to discover selective small molecule kinase inhibitors, and we obtain in-silico (see Figure~\ref{fig:msel}), as well as assay verified results (see Figure~\ref{fig:mcul-wetlab}).

%In this paper, we address this problem by designing a scalable algorithm that reweighs dataset $\S$ (using $\T$), so that training on the reweighted $\S$ leads to network weights that are close to one obtained by training on $\T$. Moreover, we apply it to a drug discovery application and obtain assay verified results.

% Moreover, intuitively, the optimal training hyper-parameters for the combined classification may be very different from the original classification task (for $\mathcal{S}$), especially if the data labels have no local smoothness i.e., points with similar labels are clustered together. However, if the data-set labels are locally smooth, then a (soft) multi-objective classification problem into classes $A\times B$ can indeed be solved without a substantial loss of accuracy in the original classification task (for $\mathcal{S}$), in principle.

Our main contributions are theoretical. Algorithm~\ref{algmain} reweighs the labeled training data-set $\mathcal{S}$ using the data-set $\mathcal{T}$ so that if we train a neural network on the reweighed $\mathcal{S}$ for a long enough period of time then the network weights will be "tilted" so that the classification error with respect to $\P_\T$ will be reduced. The amount of reduction is determined by the choice of tilt parameter $\alpha$ in Algorithm~\ref{algmain}. 
Theorems~\ref{wassthm1},~\ref{thm:main-eff1},~\ref{main11:thm}, and~\ref{thm:randsamp1} formally show correctness and efficiency of Algorithm~\ref{algmain}. In particular, 
\begin{itemize}
\item Theorem~\ref{wassthm1} justifies the choice of Wasserstein metric in Algorithm~\ref{algmain},
\item Theorem~\ref{main11:thm} shows that greedy algorithm has a sub-logarithmic approximation guarantee for minimum weight metric bipartite matching on the boolean hypercube, when the underlying dataset has small metric entropy (roughly equivalent to existence of a small covering). In~\cite{tarjan}, the authors showed that the greedy algorithm has poor approximation guarantees for computing minimum weight metric bipartite matchings.\footnote{At its core, computing Wasserstein distance is equivalent to computing minimum weight metric bipartite matchings, see for example~\cite{pankaja}.} Therefore, we need to assume and exploit some property of our input instances to get around the lower bound in~\cite{tarjan}. Assuming that input instances have small metric entropy turns out to be sufficient. This connection between small coverings and greedy matching algorithms is somewhat surprising, since we are not aware of results obtaining sharper guarantees on approximate minimum weight matching, based on the covering properties of the input data-set.
\item Theorem~\ref{main11:thm} and Theorem~\ref{thm:main-eff1} provide (multiplicative) approximation guarantees for the greedy and randomized greedy algorithm for approximate $1$-Wasserstein distance computation. 
\end{itemize}

{\em Organization:} In Section~\ref{sec:rel}, we discuss prior work from the areas of machine learning theory, algorithms and computational drug discovery, relevant to our paper. In Section~\ref{sec:stmt}, we present Algorithm~\ref{algmain} and provide a technical overview of our paper -- how our various theoretical results fit together. In particular, Theorem~\ref{wassthm1} explains why the Wasserstein metric is an intuitive and appropriate choice of metric for Algorithm~\ref{algmain}; Theorems~\ref{main11:thm},~\ref{thm:randsamp1} and~\ref{thm:main-eff1} show that the greedy random sampling based algorithm can compute the minimum weight bipartite matching, and hence Wasserstein distance, in near linear time; and they provide an upper-bound on the approximation error under our low metric entropy assumption. Thus showing that Algorithm~\ref{algmain} is scalable. In Section~\ref{dd:sec}, we describe an example application to drug discovery. %More theoretically inclined readers may skip this section. 
Finally, the supplement describes the formal setup, theorem statements, proofs; and provides further figures and details about our drug discovery application.

\section{Related work}\label{sec:rel}

The related question of learning with differing test and train distributions has been well investigated in the machine learning community under various names, including domain shift and distribution shift, see for example the books~\cite{dataset-book} and~\cite{dabook}, the papers~\cite{shimo,rosenbaum, dudik, beckel-schaef, huang, shai} and~\cite{cc}, to name just a few. The question is also relevant to our paper since our algorithm can be used to reweigh the train data-set to bring the post training neural network weights closer to what they would have been, had we trained based on the test distribution. Some prior results rely on estimating the train and test distributions. For example, parameter estimation of the densities followed by change of variable using the Jacobian. Assuming a logarithmic number of features, that leads to a $\tilde{O}(n^2)$ algorithm for distribution skew correction ($n$ being the training and test data-set size) -- much more efficient than Wasserstein distance computation that requires solving a $\Theta(n^2)$ sized linear program. %That explains why the Wasserstein distance based ideas are less explored in this context, so far. 
However, in high dimensional feature spaces, the number of samples required increases exponentially in number of features, for any formal guarantee for density estimation (as can be seen from large deviation bounds~\cite{dz}). Moreover, if we are only interested in partial tilting of one distribution towards another, then it is reasonable to look for approximate but efficient computation of Wasserstein distance. That is what we do in this paper using Algorithm~\ref{algmain}. Algorithms for distribution shift correction are applied to domain adaptation as well, as the two problems are very similar. The amount of literature in domain shift and distribution shift is vast. However, the only prior theoretical works involving Wasserstein distance computation that we found in this area were~\cite{courty} and~\cite{LeDo}, which focus on exact solution of the Wasserstein distance problem.
%which is a $\tilde{O}(n^2)$ time algorithm as well. We now have the added advantage of a provable upper bound on the test-train prediction error, before and after tilting using Algorithm~\ref{algmain}, from Theorem~\ref{wassthm1}.

The problem of efficient Wasserstein distance computation has also received much attention in the algorithms community. The paper~\cite{pankaja} studies the equivalence between Wasserstein distance computation and matching algorithms in the metric space setting. Efficient matching algorithms have been well studied in literature for five decades. The optimal algorithm for computing weighted matchings is due to Gabow and Tarjan~\cite{gabow} and runs in time $O(m\sqrt{n})$, where $m$ is the number of edges and $n$ the number of vertices in the graph. Since then more sophisticated algorithms have been designed, see for example~~\cite{vaidya, hyper, pankaja} and~\cite{indyk} to name a few. However, under our assumptions even the simple greedy algorithm performs remarkably well, and it scales efficiently for large training data-sets. 

Note that~\cite{tarjan} showed that the greedy algorithm has an abysmal approximation ratio of $n^{\log_2 3/2}$ for bipartite graphs. In this paper, we show in Theorem~\ref{main11:thm} that the approximation ratio of the greedy algorithm is much better under our bounded metric entropy assumption than the lower bound in~\cite{tarjan}. Hence, an assumption about a covering property of the input leads to more optimal matchings  -- a somewhat surprising algorithmic result that may be of independent interest.

In the context of approximate Wasserstein distance computation, we are also aware of the Sinkhorn algorithm. Sinkhorn distance computation (gradient ascent) can take about a second for 2K points,\footnote{For a comparison of various Sinkhorn distance algorithms, see Figure 4 in~\cite{sk1}.} which is likely going to be slower than most implementations of the greedy matching algorithm (sorting). Sinkhorn accuracy can be traded-off with computation time. However, as accuracy is decreased, the issue of worst case approximation guarantee becomes relevant. We are aware of additive approximation guarantees for Sinkhorn (see Theorem 1 in~\cite{sk2}) but not multiplicative ones. Directly interpreting their result, in our setting (the $d$ dimensional hypercube), their additive guarantee is: $\Theta(d\log d)$, while the diameter of the space is $d$. Hence the trivial upper-bound on the Wasserstein distance is $O(d)$, which makes their guarantee not useful for us. We are not aware of worst case multiplicative approximation guarantees for the Sinkhorn algorithm. We show worst case multiplicative approximation guarantees for the greedy algorithm in Theorem~\ref{thm:main-eff1}. The approximation factor can be informally summarized as: $O(d^c)$ for some constant $c<1$. The exact $c$ achieved depends upon the size and radii of the balls in the optimal covering, which may be unknown, and an upper bound is used in the statement. Obtaining guarantees for the Sinkhorn algorithm, under the small covering assumption here, is an interesting open problem.

One can use different objectives in the problem setting, including reweighing to minimize the generalization error.
Our objective is to bring the limiting distribution of neural net weight parameters closer.\footnote{It is worth noting here that while it is possible to construct examples (like the exponential function) where small weight perturbations lead to large perturbations in output; such (non-robust) neural nets are less likely to be useful.
For example, robustness (to weight perturbations) is desirable for model compressibility (see the discussion in section "Related works" of~\cite{tsai}). Moreover, generalization error of SGD is upper-bounded by the sensitivity (opposite of robustness) of the square of the gradient under weight perturbations (see Theorem 1 in~\cite{neu}).
Thus the distribution of the invariant measures of weights being close is a reasonable metric for neural net weight parameters; as robustness to weight perturbations is desirable for neural nets, for reasons above. } 
In~\cite{neu}, they explore upper-bounds on the generalization error in terms of the statistical properties of stochastic gradient descent (SGD).\footnote{Their bound is in terms of the variance of the gradients around the local minima. We suspect that using it together with the bound in Theorem~\ref{wassthm1} in this paper, would lead to a similar bound on the generalization error (in the limit as $t\to\infty$ and step-size is small.) but now as an expectation over network weights, where the expectation is with respect to the invariant measure of the limiting SDE corresponding to the SGD.} We do not know of any bounds in the reverse direction, but it seems to be a harder problem to judge how far apart are the network weights given the generalization error is small or large.
Discrepancy minimization for reducing generalization error has been explored recently in the context of a very similar problem to ours (see~\cite{awasthi}, and the references therein). Such approaches usually solve a convex program, while we solve a nearest neighbor problem. The latter is more tractable for large datasets. Moreover, such approaches often do not give an approximation guarantee, but we do. Furthermore,~\cite{awasthi} does not take into account the training algorithm, i.e., that we used stochastic gradient descent (SGD) to train our model. The choice of SGD for training motivated our choice of Wasserstein distance in Algorithm~\ref{algmain} (see Theorem~\ref{wassthm1}).

Finally, the idea of using deep learning for drug discovery has gained popularity in pharmaceutical research over the last few years, especially given the amount of data now available~\cite{nat-drug-disc}. The paper~\cite{pfr} shows that neural nets can be trained on DNA encoded chemical libraries to identify new small molecules that bind to a given protein target. It is particularly relevant to this work, as we build upon that. Our work extends their work by allowing us to select molecules that bind to one protein target and not to another. Other papers in this rapidly growing area include~\cite{kearnes, kearns16} and~\cite{gilmer17a}.

\section{Problem statement and overview of results}\label{sec:stmt}

Suppose we are given two training data sets, say $\mathcal{S}$ and $\mathcal{T}$, consisting of $\Theta(n)$ points each, for two different classification tasks. Assume that the labels of $\mathcal{S}$ are known, and the labels of $\T$ are unknown.\footnote{Or the labels of $\T$ may be known, but $|\T\cap\S|$ is small; in either case simply computing a weighted average of $\P_\S$ and $\P_\T$ is not possible.} Moreover, let's also assume that the datasets consist of discrete points that are a subset of a $d=\Theta(\log n)$ dimensional boolean hypercube. The points of $\S$ and $\T$ are weighted according to probability distributions, say $\P_\mathcal{S}$ and $\P_\mathcal{T}$ respectively. Our goal is to train a neural network classifier using the labeled dataset $\S$ with a reweighed distribution $\P_\S'$ (instead of using $\P_\S$), so that for training using the basic Stochastic Gradient Descent (SGD)~\citep{robins} with mean squared error loss, the limiting distribution\footnote{In the limit as training time goes to $\infty$ and step size goes to $0$.} of neural network weight parameters is closer (in $\ell_1$ distance) to the one that would be obtained had we trained using $\T$ and $\P_\mathcal{T}$.

Furthermore, in general, one may not want to reweigh $\P_\S$ so that the resulting trained neural net behaves as if trained on $\P_\T$, but only tilt $\P_\S$ towards $\P_\T$ to achieve part of that effect. For example, reweigh $\P_\S$ to a distribution $\P_\S'$ so that the post-training weight parameters of the neural network have the limiting distribution that would have resulted from the training set weights set to $(1-\alpha)\P_\S+\alpha\P_\T$. This is the case with our drug discovery example.

The boolean hypercube assumption is critically used in the proof of Theorem~\ref{main11:thm} (in Lemma~\ref{sl1}). But, it is also worth noting here that restricting the state space to a boolean hypercube is not as limiting as it may seem. First, it is natural for some applications (like drug discovery where molecules are represented as binary fingerprints). Second, the main constraint in the boolean hypercube assumption is that the underlying metric becomes $\ell_1$  (as binary strings can encode most reasonable inputs after discretization). However, it is well known (starting from~\cite{bourgain}) that one can embed arbitrary metrics into $\ell_1$ with low distortion. Thus, even if the domain is not the hypercube, one can preprocess and embed it into the hypercube (with a small loss of approximation factor) and this would be better than the baseline, i.e., not reweighting. Even without any preprocessing, if there is intuitive reason to believe that the distortion will be low (the input metric is already close to $\ell_1$), then it is likely worth reweighting using the greedy algorithm, than not reweighting.

\subsection{Drug Discovery example}
To illustrate with an example, for our drug discovery application in Section~\ref{dd:sec}: the labeled training set $\S$ consists of a subset of small molecules that are binders and non-binders for the protein MNK2, and the set $\mathcal{T}$ consists of a subset of small molecules\footnote{As an aside, each small molecule is usually mapped to a 2K character long binary string (fingerprint) of features. Thus, in this context, one may think of the underlying space of small molecules as a subset of the boolean hypercube in dimension 2K.} labeled non-binders (non-hits) for the protein MNK1. Here the labels of the molecules in $\mathcal{T}$ are known but not necessarily on the same molecules as $\mathcal{S}$ (since the data is collected at different times with different assays). The corresponding weight distributions $\P_\S$ and $\P_\T$ may be assumed to be uniform distributions supported on $\S$ and $\T$ respectively.

Given a new set of small molecules, one now wants to rank them so that first and foremost it is a binder for MNK2 and within that ranking we also want the non-binders for MNK1 to rank higher. Such models can allow us to make predictions on large commercially available catalogs and enrich compounds that have high likelihood to bind to MNK2 but not MNK1. 

Therefore, we train a neural net using $\S$ to predict binders and non-binders to MNK2, but at the same time we want to take into account the binder and non-binders to MNK1. One way to accomplish this is to reweigh the labeled example points in $\P_\S$ using $\P_\T$, so that points in $\S$ close to those in $\T$ receive higher weight in the reweighed distribution, denoted $\P_\S'$. Training the neural net on $\P_\S'$ should then achieve our goal. However, the question is how much to change the weight of each point in $\S$, especially when the datasets $\S$ and $\T$ can be huge. Our algorithm below suggests one scalable approach to the problem. This formulates the drug discovery application as an instance of our problem statement. See section~\ref{dd:sec} for details.

\subsection{Reweighing algorithm}

The reweighing algorithm (Algorithm~\ref{algmain}) solves the problem above of how to reweigh each point in $\S$ so that the trained neural net behaves as if trained on $(1-\alpha)\P_\S+\alpha\P_\T$.
\begin{algorithm}[htb]
\caption{Reweigh Distribution and Train}\label{algmain}
\begin{algorithmic}[1]
\STATE {\bf Input:} Two data-sets: $\mathcal{S}$ and $\mathcal{T}$ of size $n$ each, points weighed according to $\P_\mathcal{S}$ and $\P_\mathcal{T}$ respectively, and a tilt factor $\alpha\in[0,1]$.
\STATE {\bf Output:} Compute a distribution $\P_\S'$ on $\S$ such that the invariant distribution of network weights of a neural net model, trained using SGD with dataset $\S$ and weights $\P_\mathcal{S}'$, will be closer (in Wasserstein metric) to the invariant distribution of network weights of a neural net model trained using SGD on $\T$ with points weighted as $\P_\mathcal{T}$.\\
$\triangleright$ {\bf Algorithm starts:}\\
$\triangleright$ RandomSample returns an empirical probability distribution computed from sample size $m$.\\
$\triangleright$  $R_\mathcal{S}\subseteq\mathcal{S}$ and $R_\mathcal{T}\subseteq\mathcal{T}$ denote the random sample of points from their respective ground sets.
\STATE $\P_{R_\mathcal{S}}:=\mathrm{RandomSample}_m(\S,\P_\mathcal{S})$ 
\STATE $\P_{R_\mathcal{T}}:=\mathrm{RandomSample}_m(\T,\P_\mathcal{T})$\\
$\triangleright$ Obtain a $\alpha$-tilted version of $\P_{R_\mathcal{S}}$ that's close to $\P_{R_\mathcal{T}}$ using greedy minimum weight metric bipartite matching algorithm, described as GreedyAlgorithm (Algorithm~\ref{alg2}) in supplement
\STATE $\P_{R_\mathcal{S}}':=\mathrm{GreedyAlgorithm}
(\P_{R_\mathcal{S}},\P_{R_\mathcal{T}}, \alpha)$\\
$\triangleright$ Obtain a reweighted version of $\mathcal{S}$ 
\STATE $\P_\mathcal{S}' = (1-\alpha)\P_\mathcal{S} + \alpha\P_{R_\mathcal{S}}'$.\\
$\triangleright$ Train neural net on $\P_\mathcal{S}'$.
\STATE Use stochastic gradient descent (SGD) to train the neural net using $\P_\mathcal{S}'$.
\end{algorithmic}
\end{algorithm}

\subsection{Theoretical results}\label{theory:sec}
The rest of this section concentrates on providing a theoretical explanation for why Algorithm~\ref{algmain} should work as intended and scale computationally.
%Throughout we assume a mean square error loss function, but our discussion and proofs should extend to other similar kinds of losses as well.

\subsubsection{Choice of $1$-Wasserstein metric}
In our problem statement, a difference in $\P_\mathcal{S}$ and $\P_\mathcal{T}$ results in a difference in the convergence point of the weights in any neural net training procedure, like SGD. This is because the loss functions in the SGD algorithm will differ in the weights of their summand terms, even though they may have the same form. Therefore, given two mean squared error loss functions weighted with different probability distributions, say $\P_\mathcal{S}$ and $\P_\mathcal{T}$, on each of their terms, a natural question is: what is the relation between the limiting distribution of network weight parameters of two neural nets that are trained using the two differently weighed loss functions?

Our first theoretical contribution, Theorem~\ref{wassthm1}, shows that $W_1(\P_\mathcal{S},\P_\mathcal{T})$, the 1-Wasserstein distance between the loss function weight distributions $\P_\mathcal{S}$ and $\P_\mathcal{T}$, upper-bounds the total variation distance between the invariant measures of two such neural nets under a covariate shift like assumption. Note that the covariate shift assumption is used in distribution shift correction literature, see for example~\cite{beckel}. More formally:

\begin{assumption}\label{covshft1:assm}
%We assume that $W_1(\P_\S,\P_\T)=\Omega(1)$. 
Let $f(w,x)$ be the neural network output, for weights $w$, input $x$ and corresponding label $y$. We assume that $x\in Q_d$ (the $d$ dimensional boolean hypercube), $f$ and $y$ are bounded, $y,f(w,x)\in[0,1]$ and for all $w$:
\begin{equation}
{\scriptstyle
\left|\E_{y\sim\P_\S(\cdot|x)}[(y-f(w,x))^2] - \E_{y\sim\P_\T(\cdot|x)}[(y-f(w,x))^2]\right|=O(1).
}
\end{equation}    
\end{assumption}
Assumption~\ref{covshft1:assm} is weaker than the usual covariate shift assumption, since in the latter case, the RHS of Equation~\ref{covcond1} would equal $0$.\footnote{The assumption can be further weakened by not requiring that it needs to hold for all $w$, but only $w$ near local optima; see the discussion in the supplement Section~\ref{wass:sec}.} Under Assumption~\ref{covshft1:assm}, we show the following theorem:
\begin{theorem}(see Theorem~\ref{wassthm} for precise statement)\label{wassthm1}
% In the setting defined by Equations~\ref{covcond},\ref{losseq} and~\ref{diffeq} above, 
Suppose we train two neural networks, such that (1) the limiting stochastic differential equation (SDE) corresponding to the training SGD (as SGD step-size goes to $0$) is strongly elliptic,\footnote{This ensures the invariant measure of the SDE exists, is smooth and unique.} and (2)  Assumption~\ref{covshft1:assm} holds, on different input distributions, $\P_\T$ and $\P_\S$, using the stochastic gradient descent (SGD) algorithm. Then, if  $W_1(\P_\S,\P_\T)=\Omega(1)$ (the interesting case of our problem), the total variation distance between their invariant measures can be bounded by $O(W_{1}(\P_\T,\P_\S))$, for the limiting SDE of the SGD.
\end{theorem}
Therefore, the above explains our choice of the $1$-Wassertein metric as the metric to use in the greedy minimum weight metric bipartite matching computation in Algorithm~\ref{algmain}.
%(see also Remark~\ref{rmk:levy} for comparing with the Levy metric).

\subsubsection{Metric bipartite matching}
The next question is, suppose we want to compute a distribution $\P_\mathcal{S}'$ with set of support $\S$, such that it minimizes $1$-Wasserstein distance between $(1-\alpha)\P_\mathcal{S} + \alpha\P_\mathcal{T}$ and $\P_\mathcal{S}'$, for some fixed choice of tilt factor $\alpha\in[0,1]$.\footnote{
The optimum value of $\alpha$ can be chosen by trial and error after running multiple training and validations, to reduce any over-fitting.
} 
We would then use $\P_\mathcal{S}'$ as the new set of weights for neural net training. 

While the optimal $\P_\S'$ mentioned above can be computed by solving a linear program (LP) that closely resembles the $1$-Wasserstein distance computation LP, the number of constraints would be quadratic in the size of the data-sets, making the computation intractable for large datasets.\footnote{A typical large data-set has 10-100M examples, and computing $W_1$ over two such data-sets requires solving a linear program -- a $\Theta(n^3)$ time procedure, resulting in the order of $10^{24}$ computational operations!} Therefore, we look for inaccurate but efficient algorithms and a natural candidate is the randomized greedy algorithm for minimum weight metric bipartite matching, for reasons explained below.

Given a bipartite graph with vertices embedded in a metric space, the {\em metric minimum weight bipartite matching problem} asks to compute a minimum weight matching, where the weight of a matching is the sum of the lengths of edges in the matching.

The 1-Wasserstein metric has an equivalent interpretation as an optimal transport problem. In fact, it is equivalent to solving the metric minimum weight bipartite matching problem~\cite{tarjan}. The reduction is fairly intuitive and consists of duplicating the supply and demand points, in the optimal transport problem, in proportion to their weights in the dataset. Theorem~\ref{thm:ps} (essentially repeated from~\cite{pankaja}) provides a formal statement reducing the former to the latter.

One tractable way to compute a minimum weight bipartite matching is to use a faster but sub-optimal algorithm. The greedy algorithm, formally studied by~\cite{tarjan} in this context, is a natural contender. In this paper, we show that if our input instances admit a small sized covering then the greedy algorithm, i.e., Algorithm~\ref{algmain}, provides a better approximation guarantee than the worst case lower bound from~\cite{tarjan}. Our main contribution here is Theorem~\ref{thm:main-eff1}, which is just a conjunction of Theorems~\ref{main11:thm} and~\ref{thm:randsamp1}. Theorem~\ref{main11:thm} provides the combinatorial argument around the for the greedy algorithm under the small covering assumption, and is discussed in Subsection~\ref{smallcov:sbs}. Theorem~\ref{thm:main-eff1} can be informally stated as follows.
\begin{theorem} (see Theorem~\ref{thm:main-eff} for general statement with trade-offs)\label{thm:main-eff1}
Suppose we are given two data-sets with $\mathcal{S}$ and $\mathcal{T}$ that are weighted according to distributions $\P_\mathcal{S}$ and $\P_\mathcal{T}$. If, 
\begin{enumerate}
\item Small covering: $\S\cup \T$ admits a covering with $\eta$ $\ell_1$-balls of radius $\zeta$; with $\eta,\zeta=O(\log^c n)$, and $c \le \frac{1}{2(1+\log_2(3/2))}$; and
\item $\P_\S$ and $\P_\T$ are sufficiently far apart: $W_1(\P_\mathcal{S},\P_\mathcal{T})\ge\log\log n +S_n$, where $S_n$ measures how spread out the the covering balls are (see Definition~\ref{dfn:spread}).
\end{enumerate}
then the greedy algorithm achieves an approximation ratio of $O(d^{c'})$ for some $c'\le 0.73$,\footnote{The value $0.73$ comes from using $\xi=2$ in the bound in Theorem~\ref{main11:thm}.} with probability $1-o(1)$, when computed on a small random sample of $r(n)$ fraction of data-points and $r(n)\to 0$.

%The greedy algorithm run on a small subset obtained via a random sample of an $n$ point data-set can be used to approximate the $1$-Wasserstein distance with a poly-logarithmic factor approximation, as long as the dataset admits a small covering of poly-logarithmic size. 
\end{theorem}

\subsubsection{Small covering assumption}\label{smallcov:sbs}
So, the question arises: What does a small covering assumption above mean in the context of minimum weight metric bipartite matching algorithms, and what is its underlying combinatorial connection to such matchings?

One way to specify small coverings is via {\em metric entropy} -- the minimum number of balls of a given radius required to cover the point set (see Definition~\ref{dfn:me}). It turns out that for computing minimum weight bipartite matchings on pointsets with low metric entropy, the greedy algorithm of~\cite{tarjan} performs provably well. 
\begin{theorem} (see Theorem~\ref{main1:thm})\label{main11:thm}
For $d=\Theta(\log n)$, let $\eta = O(d^{\frac{1}{\xi\log_2 3/2}})$ and $\zeta=O(\eta)$. Suppose that the input points can be covered by $\eta$ $\ell_1$-balls of radius $\zeta$, then the greedy algorithm achieves an approximation factor of $\max\{2\zeta, O\left(d^{\frac{1+\xi\log_2(3/2)}{\xi(1+\log_2(3/2))}}\right)\}$ for $\xi>1$, on points on the $d$ dimensional hypercube. 
\end{theorem}

Note that for points on the $d$ dimensional hypercube, an approximation factor of $d$ for minium weight metric bipartite matching is trivial, but we improve it to $O(d^{c})$ for $c < 1$, exact $c$ depends on the metric entropy of the dataset (see Corollary~\ref{cor:nontriv} and Theorem~\ref{main1:thm} for precise trade-offs).

The crux of the proof of Theorem~\ref{main11:thm} consists of a structural characterization of alternating cycles\footnote{An {\em alternating cycle} in a matching is simply a cycle consisting of alternate matched and unmatched edges.} and matchings (Lemma~\ref{sl1}) that may be of independent interest. The key to the proof of Lemma~\ref{sl1} is the following (informal) idea: Let $\gamma$ be an alternating cycle in the greedy matching. Suppose that the sum of weights of the matched edges between vertices in $\gamma$ is $\kappa$ times their weight in the minimum weight matching, i.e., the cycle $\gamma$ is long. But, how can a cycle be long in a metric space like the hypercube, which has diameter $\log n$, and number of vertices $n$? The answer is that the cycle $\gamma$ must have many long edges. That together with the assumption that there exists a small covering implies a contradiction for an appropriate choice of $\kappa$. 
% Besides Theorem~\ref{main1:thm}, the proof of Theorem~\ref{thm:main-eff} relies on Theorem~\ref{thm:randsamp}. 
% Theorem~\ref{main1:thm} show that greedy minimum weight matching on bipartite graphs for vertex sets with low metric entropy has a much better (poly-logarithmic) approximation guarantee in our case, as opposed to the polynomial approximation guarantee from~\cite{tarjan}. This argument, especially the connection between covering and matching in Lemma~\ref{sl1} may be of independent interest.

\subsubsection{Greedy on random sample}
Algorithm~\ref{algmain} uses the greedy algorithm on top of a small random sample of the datasets to deal with the quadratic time complexity when datasets $\S$ and $\T$ are large. Therefore, we show that using a small random sample does not lead to large deterioration in the approximation guarantee.
Theorem~\ref{thm:randsamp1} states that using random samples for datasets with bounded metric entropy do not lead to a much worse approximation guarantee, if the $W_1$ distance between the empirical distribution computed from $m$ samples ($\hat{\mu}_m$) and the true distribution ($\mu$) is known to be not too small (which is the interesting case of the problem).

\begin{theorem} (Informal; see Theorem~\ref{thm:randsamp} for precise formal statement)\label{thm:randsamp1}
For a dataset with $\log^{O(1)}n$ metric entropy with $\ell_1$ balls of $\log^{O(1)}n$ radius, and a random sample of size $m=o(n)$, the 1-Wasserstein distance between the empirical distribution and the true distribution of data-sets with bounded metric entropy obeys the following Sanov type concentration bound:
\begin{equation}
{\scriptstyle
      \exists\ m=o(n),\ \lim_{n\to\infty}\frac{1}{m}\ln \mathbb{P}\left(W_1(\hat{\mu}_m, \mu) \ge \log\log n+ S_n\right) \le -\Omega(1),
}
\end{equation}
where $S_n$ and it measures how spread out the the covering balls are (see Definition~\ref{dfn:spread}).
\end{theorem}

Note that, $S_n\le\log n$, for the hypercube, so when $S_n=o(\log n)$ but the diameter of our point set is $O(\log n)$, i.e., most points are clustered around some $x_0$ except for small fraction of outliers, then one can use a random sample to reduce the input size further without deterioration of the worst case approximation factor.

For the proof of Theorem~\ref{thm:randsamp1}, we need large deviation bounds for the $1$-Wasserstein distance between the theoretical distribution and its empirical distribution. Such results have been explored previously with tight Sanov's theorem type bounds in low dimensional spaces (see for example~\cite{bgv}). However, our underlying space has large dimension, i.e., $\log n$, which depends upon $n$. The constants in the exponential in the theorems in~\cite{bgv} will thus depend on $n$, and it's not immediately clear to us whether the dependence can be easily removed. Hence, we need the assumption of low metric entropy for the same results to go through (see chapter 6 in~\cite{dz}).

\section{Example application: drug discovery}\label{dd:sec}

A natural question is, when does the small covering assumption hold in practice? This seems to happen in the drug discovery setting.\footnote{The combinatorial synthesis process utilized in DNA encoded library (DEL) compounds often results in local chemical similarity among compounds that share common building blocks. Since similar molecules likely have the same binding behavior, synthesized molecules form a small ball around a parent molecule in the molecule fingerprint space. Therefore, molecule binding vs non-binding data-sets likely have low metric entropy.\label{fndellabel}} So, as a concrete motivating example, we illustrate an application of Algorithm~\ref{algmain} to a toy problem in the drug discovery area (see also Section~\ref{sec:expt}).

In drug discovery, typically, one wants to isolate {\em selective} small molecules (inhibitors) that bind strongly to a given enzyme, but often we want to exclude small molecules that bind to another similar enzyme. For example, MNK1 and MNK2 are two structurally similar kinases (a kinase is an enzyme for phosphorylation or de-phosphorylation of proteins). We want to identify small molecules that bind strongly to MNK2 (MNK2 hits), but we also prefer that the identified small molecules not bind to MNK1 (MNK1 non-hits). In other words, we want to isolate molecules that are selective for MNK2 over MNK1. 

{\em Performance improvements:} In the in-silico experiments, we were able to increase the percentage of MNK1 non-hits in our set of top predicted MNK2 hits -- the selectivity -- from 54\% to 95\% on holdout data, using the reweighing procedure in Algorithm~\ref{algmain}. We are not aware of other such multi-target prediction results in DNA encoded library (DEL) space (see~\cite{satz} for background), where one simultaneously predicts hits/non-hits against two or more proteins. However, the success rates for single target experiments with traditional high-throughput screening is $\sim1\%$ (see for example the discussion in~\cite{pfr}) and it is generally accepted that multi-target prediction is a harder problem. 

{\em Training:} We used a relatively small training set of about $250$K small molecules in total; labeled as MNK1 non-binders, and MNK2 binders as well as non-binders. To evaluate the effect of reweighting on the selectivity\footnote{Recall that, we want to find small molecules that are MNK2 binders but preferrably MNK1 non-binders.} of predicted MNK2 binders, we use a small holdout set of $7$K small molecules that consists of molecules which are labeled as: MNK2 hits (binders), and MNK1 hits (binders) or MNK1 non-hits (non-binders). In Figure~\ref{fig:msel}, for the neural network models with and without reweighing, we plot the cumulative number of MNK1 non-hits on the $y$-axis; and on the $x$-axis any given point, say $k$, represents the top $k$ predicted MNK2 hits from the examples in the holdout set. 
\begin{figure}[htb]
\centering
\includegraphics[scale=0.35]{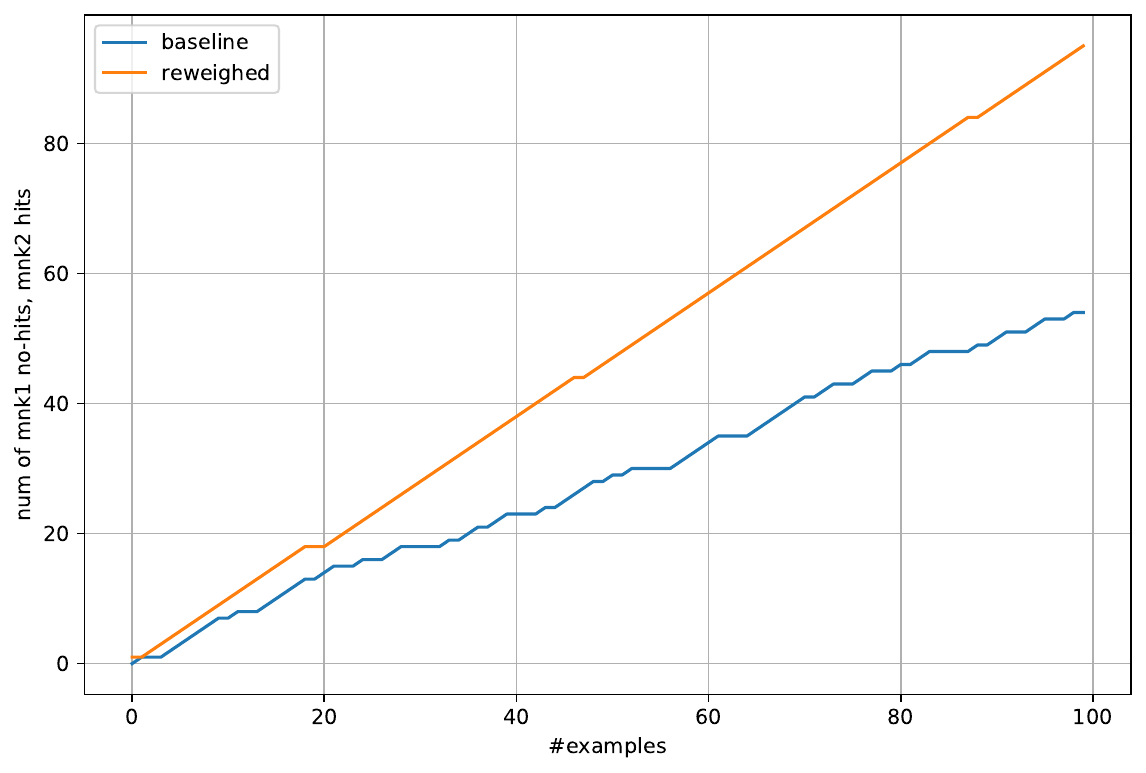}
\caption{Selectivity of reweighed (using Algorithm~\ref{algmain}) and baseline (without reweighing) neural nets. Note that this increase in selectivity from 54\% to 95\% came without any significant change in the validation loss -- the AUC for the classification of MNK2 binders vs non-binders remained around $0.6$ in both cases.}
\label{fig:msel}
\end{figure}
While we can not make our training data-sets and code public for proprietary reasons, we were able to experimentally (in wet-lab) verify that two out of the top fifty (actually 43, since 7 out of 50 molecules could not be synthesized and tested) predicted selective small molecules, obtained by running our neural network model on the Enamine 1.9B molecules catalog (\url{https://enamine.net}), were indeed selective for MNK2 over MNK1.\footnote{The compounds are Z1918489591 and Z5890616727 in the enamine catalog.} That is a success rate of roughly 5\% on this admittedly small sample set. We do note that the results are from a single point concentration assay and can be noisy.

{\em Assay based validation:} More importantly, the two predicted and assay tested molecules in Figure~\ref{fig:mcul-wetlab} provide a degree of verification for our experimental application (which has been the motivation, but is not the focus of this paper). 
\begin{figure}[htb]
\centering
    \includegraphics[scale=0.2]{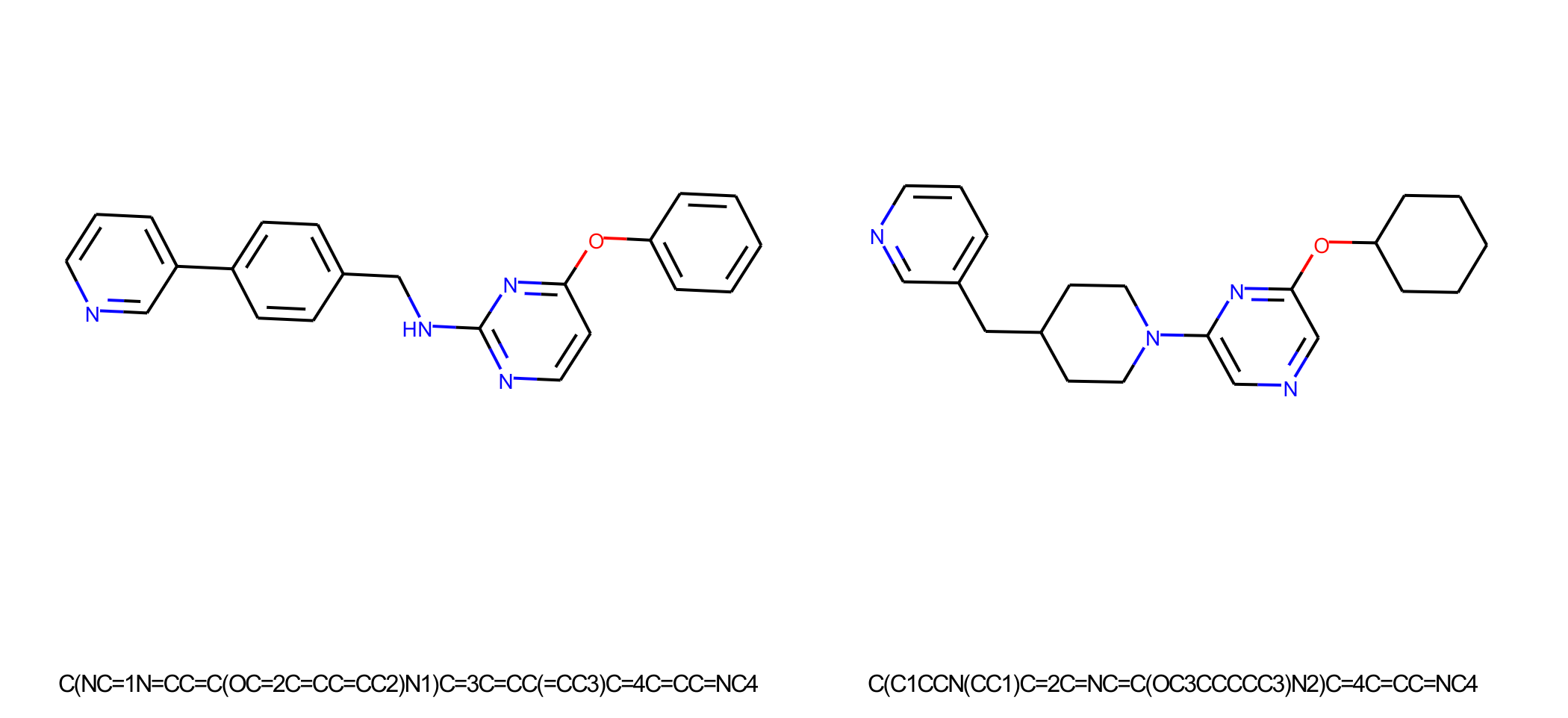}
\caption{Two predicted and verified selective MNK1 non-hits and MNK2 hits from the Enamine catalog. The enzyme activity was found to be above 50\% for MNK1 but below 50\% for MNK2 at $10\mu$M concentration for each of the two small molecules: $\sim20\%$ vs $70\%$ and $\sim39\%$ vs $59\%$. Note that these values are from single point concentration assay and can be noisy.}
\label{fig:mcul-wetlab}
\end{figure}
%The supplement summarizes the results of the experiments and provides further context about them.

{\em Effect of reweighing:} The main aim of the following discussion is to justify Remark~\ref{rmk:mnk1-mnk2} below.

\begin{remark}\label{rmk:mnk1-mnk2}
Algorithm~\ref{algmain} is successfully highlighting a closely packed portion of the chemical space of MNK2 binders that are non-binders to MNK1, while the top predicted MNK2 binders in the baseline model are relatively more spread out in the chemical fingerprint space.
\end{remark}

In order to observe the effects of reweighing on the top $100$ predicted binders to MNK2, we did four types of similarity comparisons in between two sets: (1) top $100$ baseline predictions, and (2) top $100$ treatment (reweighed) predictions:\footnote{Recall that, the baseline training dataset is not weighed using MNK1 data, so it effectively just predicts small molecules that are potent MNK2 binders without regard to their MNK1 binding property. The experiment training dataset reweighs the MNK2 binders and non-binders to highlight MNK1 non-binders. Therefore, if the reweighting is done optimally, with just the right amount of tilt, then we may be able to predict potent MNK2 binders that are preferrably MNK1 non-binders. }
\begin{enumerate}
\item For every baseline molecule we computed its mean similarity with the remaining $99$ baseline molecules,
\item For every treatment (reweighed) molecule we computed its mean similarity with the remaining $99$ treatment molecules.
\item For every baseline molecule we computed its mean similarity with the $100$ top treatment molecules,
\item For every treatment molecule we computed its mean similarity with the $100$ top baseline molecules.
\end{enumerate}

We expect the top predictions to qualitatively have the following properties:
\begin{itemize}
   \item The treatment (reweighing) highlights a small portion of the molecular space of MNK2 binders the ones that are not MNK1 binders, so we expect the mean similarity in (1) to be smaller than (2).
   \item If we assume there's a small region of molecules in the fingerprint space which are selective, i.e., bind to MNK2 but not MNK1, and from which the baseline has sampled about $50\%$ of its top $100$ predictions, while the treatment has sampled about $95\%$ of its top $100$ predictions (see Figure~\ref{fig:msel}); then we expect the following: (1) a plot of the computed similarities in item (3) above will have a bimodal distribution with about $50\%$ of the mass in each mode, and (2) a plot of the computed similarities in item (4) will have a bimodal distribution with about $95\%$ of the mass in one mode.
\end{itemize}

Indeed our experiments are consistent with the above. In our plots, we found that the top 100 predicted binders in the experimental model are indeed more similar to each other than the baseline model (mean Tanimoto similarity increases from $0.32$ to $0.53$), and hence are "packed" more closely together (see Figure~\ref{fig:cluster1}).
\begin{figure}[htb]
% \begin{subfigure}%{.5\textwidth}
  \centering
  % include first image
  \includegraphics[width=0.45\linewidth]{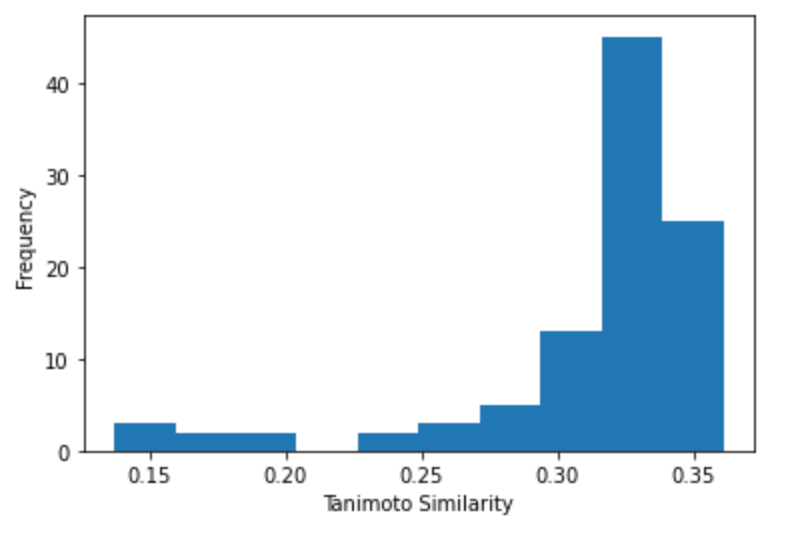}\quad
  % \caption{Clustering of top molecules in base model}
  % \label{fig:sub-first}
% \end{subfigure}
% \begin{subfigure}%{.5\textwidth}
%   \centering
  % include second image
  \includegraphics[width=0.45\linewidth]{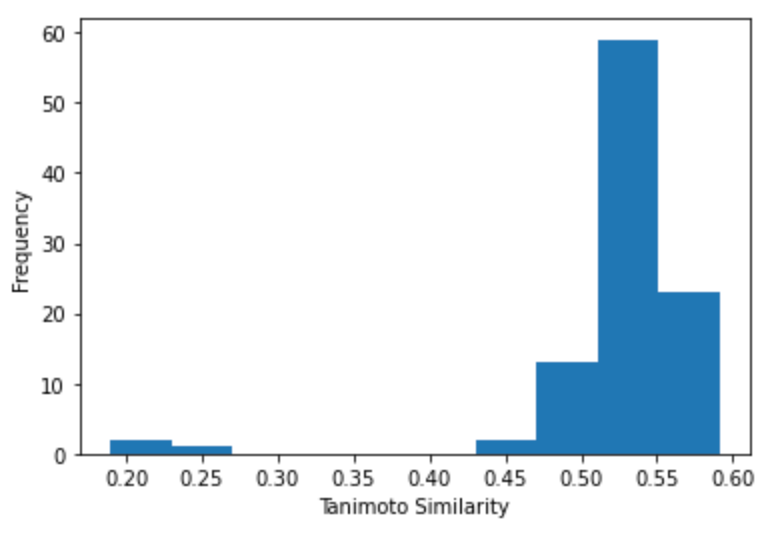}  
  % \caption{Clustering of top molecules in reweighed model}
  % \label{fig:sub-second}
% \end{subfigure}
\caption{(L) Tanimoto similarities of top molecules in base model vs (R) Tanimoto similarities of top molecules in reweighed model}\label{fig:cluster1}
\end{figure}

We also observe that the mean similarity score for the cross comparison, in Figure~\ref{fig:cluster2}(L), is clearly bifurcated into two parts with means around $0.2$ and $0.5$ respectively. The right peak of 30 small molecules mostly comes from the 54 molecules that were MNK1 non-binders. This means that about half of the top 100 predicted MNK2 binders in baseline are much more similar to the top 100 predicted MNK2 binders in the reweighed model. Thus, Figures~\ref{fig:cluster1} and~\ref{fig:cluster2} are consistent with Remark~\ref{rmk:mnk1-mnk2}.
\makeatletter
\setlength{\@fptop}{0pt}
\makeatother
\begin{figure}[htb!]
% \begin{subfigure}%{.5\textwidth}
  \centering
  % include first image
  \includegraphics[width=0.45\linewidth]{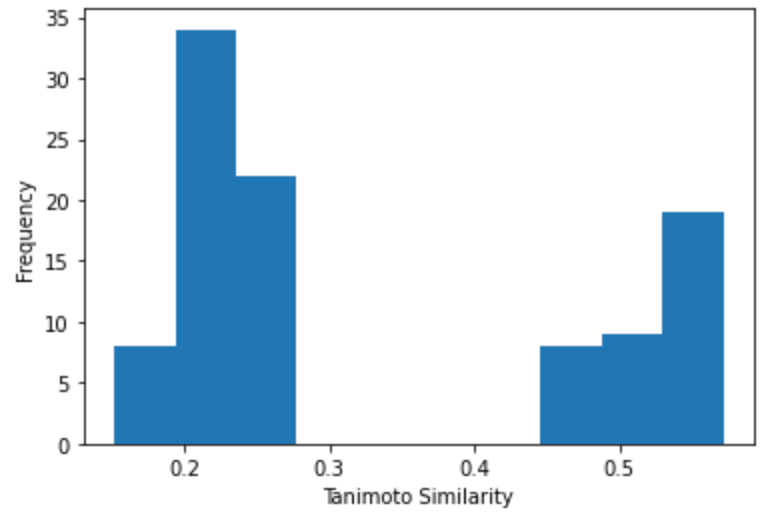}\quad 
  %\caption{Mean Tanimoto Similarity for each top molecule in base with reweighed model (note the bifurcation).}
  % \label{fig:sub-third}
% \end{subfigure}
% \begin{subfigure}%{.5\textwidth}
%   \centering
  % include second image
  \includegraphics[width=0.45\linewidth]{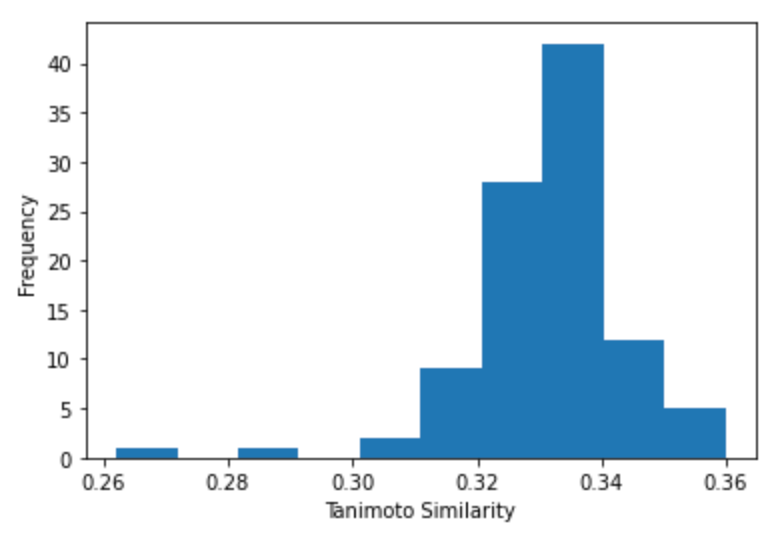}  
  %\caption{Mean Tanimoto Similarity for each top molecule in reweighed with base model}
  % \label{fig:sub-fourth}
% \end{subfigure}
\caption{(L) Mean Tanimoto Similarity for each top molecule in base with reweighed model (note the bifurcation) vs (R) Mean Tanimoto Similarity for each top molecule in reweighed with base model.}\label{fig:cluster2}
\end{figure}
\section{Acknowledgements}
The author is grateful to Wen Torng, JW Feng, Jin Xu and Partick Riley for their help and advise with this paper.
\clearpage
\bibliographystyle{plainnat}
\bibliography{main}

%%%%%%%%%%%%%%%%%%%%%%%%%%%%%%%%%%%%%%%%%%%%%%%%%%%%%%%%%%%%%%%%%%%%%%%%%%%%%%%
%%%%%%%%%%%%%%%%%%%%%%%%%%%%%%%%%%%%%%%%%%%%%%%%%%%%%%%%%%%%%%%%%%%%%%%%%%%%%%%
% APPENDIX
%%%%%%%%%%%%%%%%%%%%%%%%%%%%%%%%%%%%%%%%%%%%%%%%%%%%%%%%%%%%%%%%%%%%%%%%%%%%%%%
%%%%%%%%%%%%%%%%%%%%%%%%%%%%%%%%%%%%%%%%%%%%%%%%%%%%%%%%%%%%%%%%%%%%%%%%%%%%%%%
\newpage
\appendix
\onecolumn
% \fancyh{\bf{Enhancing selectivity using Wasserstein distance based reweighing: Supplementary Materials}}
% \hline
%\appendix 
\section{Supplement}
In this appendix, we first formally state the assumptions and outline our theorems in Section~\ref{sec:proofs}, and then provide the full proofs in Section~\ref{sec:details}. Section~\ref{sec:expt} contains some more details about the drug discovery application.
\section{Formal setup and theorem statements}\label{sec:proofs}
\subsection{Choice of metric in Algorithm~\ref{algmain}: Bounding $1$-Wasserstein distance suffices}\label{wass:sec}

In this subsection, in Theorem~\ref{wassthm}, we show that the Wasserstein distance between two measures upper bounds the total variation distance between the invariant measures underlying the stochastic gradient descent (SGD) algorithms. 

Let $X\times Y$ denote the usual space of labeled examples i.e, in our case $X\subseteq\{0,1\}^{d(n)}$ is the set of feature values and $Y:=\{0,1\}$ is the set of labels. 
Our object of interest in this section is a neural network with smooth bounded activation functions. Let $y=f(w,x)$ denote the abstraction of our neural network, where $w$ denotes the real valued vector of weights. For a depth $p$ neural-net with piecewise polynomial activation functions of degree $q$, $f(w,x)$ is piecewise polynomial in $x$ with degree at most $pq$.

Let $\ell(\cdot)$ denote the loss function, which we will assume to be the sum of square loss, for the sake of concreteness. The ideas easily extend to any low degree loss function. The training loss can be written as:
\begin{equation}\label{losseq}
    \ell_w(\P_\S): = \E_{(x,y)\sim\P_\S}[(y-f(w,x))^2].
\end{equation}

We make a {\em covariate shift type assumption}. 
\begin{assumption}\label{covshft:assm}
%We assume that $W_1(\P_\S,\P_\T)=\Omega(1)$. 
Assume that $f$ and $y$ are bounded, say $y,f\in[0,1]$, and for all $w$:
\begin{equation}\label{covcond1}
%{\scriptstyle
\left|\E_{y\sim\P_\S(\cdot|x)}[(y-f(w,x))^2] - \E_{y\sim\P_\T(\cdot|x)}[(y-f(w,x))^2]\right|=O(1).
%}
\end{equation}    
\end{assumption}
Essentially, it says the data-sets have similar average loss in the same neighborhood for a given set of weights. It is also worth noting here that, if the invariant distribution of the SGD is concentrated around the local minima, i.e., if all but $\epsilon$ fraction of the mass of the invariant distribution, for some small positive $\epsilon$, are present within a small neighborhood around the local minima; then for the proof of Theorem~\ref{wassthm1} to remain valid, Assumption~\ref{covshft1:assm} only really needs to hold for $w$'s that are close to some locally optimal $w^*$.

% assumption: ergodicity, uniform ellipticity, isotropic, covariate shift
Recall that,\footnote{See for example~\citet{diff-pc} for an introductory discussion between SDEs and limiting SGD dynamics, and see~\citet{benarous} for a more advanced discussion in this regard.} a stochastic gradient descent algorithm with loss function $\ell$ can be abstracted as the It\'{o} diffusion in the limit of small step size: 
\begin{equation}\label{diffeq}
    \mathrm{d}w_\S(t) = \nabla_w \ell_w(\P_\S)\mathrm{dt} + \sigma_\S\mathrm{dB(t)},
\end{equation}
where $\nabla_w$ denotes gradient with respect to $w$, $B(t)$ denotes Standard Brownian Motion in $|w|$-dimensions and the matrix $\sigma_\S$ depends upon the variance of the loss function for the mini-batch, mini-batch size and the learning rate.

Such a diffusion process is associated with an invariant measure or equilibrium distribution. In our context, this is the distribution of weight parameters of the neural net, as training time becomes very large. The existence of an invariant measure requires that the infinitesimal generator\footnote{The infinitesimal generator $\mathcal{L}$ of the diffusion in Equation~\ref{diffeq} may be written as: $$\mathcal{L}f(w)=(\nabla_w f)^T(\nabla_w \ell_w(\P_\S)) + \mathrm{Trace}\left((\nabla_w^2 f(w))\frac{\sigma_\S\sigma_\S^T}{2}\right).$$
It captures its most important properties, and in particular its adjoint $\mathcal{L}^*$ characterizes the invariant distribution (when it exists), i.e., $\mathcal{L}^*\rho=0$, where $\rho$ is the invariant distribution.} associated with the diffusion be well behaved. In particular, Assumption~\ref{sgd:assm} about the generator ensures the existence of a unique limiting (invariant) measure, see~\cite{unif-elp}.

\begin{assumption}\label{sgd:assm}
We assume that the diffusion corresponds to an uniformly elliptic generator. Furthermore, we assume $\sigma_\S$ is isotropic i.e, it's a scalar multiple of the identity $\sigma\cdot\mathrm{Id}$ and that $\sigma_\S=\sigma_\T$, in Theorem~\ref{wassthm}.
\end{assumption}
We relax the isotropy assumption somewhat in Corollary~\ref{cor:wassthm}. Under the technical assumptions discussed above, we can prove the following result.

\begin{theorem}\label{wassthm}
Suppose we train two neural networks, under the assumptions~\ref{sgd:assm} and~\ref{covshft:assm}, on different input distributions, $\P_\T$ and $\P_\S$, using the stochastic gradient descent (SGD) algorithm. If $W_1(\P_\S,\P_\T)=\Omega(1)$, then the total variation distance between their invariant measures can be bounded by $O(W_{1}(\P_\T,\P_\S))$,\footnote{Note that the $O(\cdot)$ here characterizes the linear dependence on $W_{1}(\P_\T,\P_\S)$, and there is a large constant factor hidden in the notation here due to the maximum eigenvalue of the inverse of the adjoint operator of the corresponding SDE.} in the limit as training time goes to $\infty$ and SGD step-size goes to $0$.
\end{theorem}

Proof deferred to Section~\ref{sec:details}.

\begin{remark}\label{rmk:levy}
For dimension $d(n)$ large, the Levy-Prokhorov distance $L(\P_\mathcal{S},\P_\mathcal{T})$ between two distributions can be $\omega(1)$ times the Wasserstein distance $W_1(\P_\mathcal{S},\P_\mathcal{T})$, so a Levy-Prokhorov metric based algorithm and guarantee can be weaker than the above. 
\end{remark}
\subsection{The Greedy Algorithm}\label{algo:sec}
\subsubsection{Reduction to bipartite matching}

In this subsection, our main contribution is Theorem~\ref{main1:thm}. So far, we have established that $1$-Wasserstein metric is a sufficient topology to work with. This leads to the problem of computing the $1$-Wasserstein distance on two large datasets. That problem is equivalent to the minimum weight bipartite matching problem. In particular, we have the following lemma from~\cite{pankaja}.

\begin{theorem}\cite{pankaja}\label{thm:ps}
Given an instance of the optimal transport problem with supply and demands on two sets of points $(R,B)$, i.e, equivalently the $1$-Wasserstein distance computation problem in our case; we can construct an instance of the minimum weight bipartite matching problem such that solving the latter up to an approximation factor $\alpha$ will solve the former up to the same approximation factor $\alpha$.
\end{theorem}

The GreedyAlgorithm (below) carries out the reduction in Theorem~\ref{thm:ps} and calls GreedyMatch which matches two multisets embedded in a metric space using greedy algorithm on the edge lengths.

{\em Notation:} Scaling a discrete probability distribution $\P$ up by an integer factor of $C$ leads to a numerical rounding error of  $\frac{1}{C\min\{\P\}}$, where $\min\{\P\}$ denotes the minimum positive value of density $\P$. Assume that we pick a large enough constant $C$ below, so that we can ignore the rounding error for the purposes of Theorem~\ref{main1:thm}.
\begin{algorithm}[htb]
\caption{GreedyAlgorithm} \label{alg2}
% (scale distributions and call bipartite matching)}\label{alg2}
\begin{algorithmic}[1]
\STATE {\bf Input:} Two probability distributions $\P_B,\P_R$ supported on $B,R\subset Q_d$, and a tilt factor $\alpha\in(0,1)$.
\STATE {\bf Output:} Probability distribution $\P_B'$ supported on $B$. Note
$\P_B'$ is close to $\alpha\P_R+(1-\alpha)\P_B$ in $W_1$, under assumptions of Theorem~\ref{main1:thm}.\\
$\triangleright$ {\bf Algorithm starts:}
\STATE For {$r\in R$}
\STATE \quad$\mathrm{Supply}(r)\gets C\cdot\alpha\P_R(r)$
%\EndFor
\STATE For {$b\in B$}
\STATE \quad$\mathrm{Demand(b)}\gets C - C\cdot(1-\alpha)\P_B(r)$
\STATE \quad If {$\mathrm{Demand}(b) < 0$}
\STATE \qquad $\mathrm{Demand(b)}\gets 0$
% \EndIf
% \EndFor
\STATE Create multi-set $B',R'$ with multiplicities of each element being equal to their Demand and Supply respectively.
\STATE Use the usual BFS (breadth first search) based greedy algorithm: GreedyMatch (Algorithm~\ref{alg1}), on sets $R'$ and $B'$ to compute the met (matched) demands, i.e., the extent to which the demands of $B$ that are actually fulfilled by $R$.
\STATE Normalize the weights of met demands to obtain a probability distribution $\P_B'$ supported on $B$.
\STATE \textbf{return} $\P_B'$.
\end{algorithmic}
\end{algorithm}

\subsubsection{Greedy algorithm and metric entropy}

Recall that the data-points are set in the $d$-dimensional hypercube $Q_d$ with $\ell_1$ metric, where $d=O(\log n)$. The minimum weight bipartite matching problem is known to be harder than its non bipartite version. For example, the greedy algorithm is known to have a lower bound of $\Theta(n^{\log_2 3/2})$~\cite{tarjan} for the bipartite version with $n$ data-set $\T$ and $n$ data-set $\S$ vertices. 
% As an aside, a variant of the greedy algorithm (the hyper-greedy algorithm) provides a $\log n$ approximation in the non-bipartite case i.e, when any vertex can be matched to any other vertex. For the bipartite case, we obtain better approximation guarantees via the greedy algorithm, assuming small metric entropy of the input point sets. 

\begin{definition}
Given a perfect matching $M$ over a subset of vertices $C$ in a graph $G$, an {\em alternating cycle} $\gamma$ is a cycle in $G$ such that each alternate edge in the cycle belongs to $M$. Note that any such $M$ corresponds to a set of vertex disjoint alternating cycles.
\end{definition}

In particular, Reingold and Tarjan~\cite{tarjan} essentially show the following theorem.
\begin{theorem}\cite{tarjan}\label{tjn}
Given a set of $n$ data-set $\T$ and data-set $n$ $\S$ points in a metric space, the greedy algorithm returns a matching with weight that is within a factor of $|\gamma|^{\log_2\frac{3}{2}}$ of the minimum weight matching, where $|\gamma|$ is the length of the longest alternating cycle $\gamma$ in the set (which can be $\Theta(n)$ for $Q_{\log n}$).
\end{theorem}

In order to improve upon their guarantee, we will exploit the following assumption for our input instance.
\begin{assumption}
We assume that all input $\T$ and $\S$ points can be covered by $\eta$ balls of radius $\zeta$ lying within $Q_d$. We call such an input instance {\em $(\eta,\zeta)$-bounded}. The parameters $\eta$ and $\zeta$ will determine the approximation guarantee of our algorithm.
\end{assumption}

\begin{definition}\label{dfn:me}
Given a metric space, say $(Q, d)$ and $E\subset Q$,, the {\em metric entropy} $N^{\mathrm{ent}}_r(E)$ is the largest number of points $\{x_1,\dots,x_n\}$ one can find in $E$ that are $r$-separated, i.e., $d(x_i,x_j) \geq r$ for all $i \neq j$.
\end{definition}

\begin{definition}
Given a metric space, say $(Q, d)$ and $E\subset Q$, the (external) {\em covering number} $N^{\mathrm{cov}}_r(E)$  is the fewest number of points $\{x_1,\dots,x_n \in Q\}$ such that the $d$-balls $\{B(x_1,r),\dots,B(x_n,r)\}$ cover $E$.
\end{definition}

\begin{lemma}[Structural Lemma]\label{sl1}
For an alternating cycle $\gamma$ induced by the greedy matching, if the weight of edges in the alternating cycle coming from the greedy matching is at least $\alpha$ times the weight of edges in the alternating cycle coming from the minimum weight matching then the metric entropy of $\gamma$ is large i.e, more precisely, 
\begin{equation}
    \left( N^{\mathrm{ent}}_{\alpha/2}(\gamma) \cdot \frac{2d-\alpha}{\alpha}\right)^{\log_2 3/2}\ge\alpha.
\end{equation}
\end{lemma}

Proof deferred to Section~\ref{sec:details}. Note that an approximation factor of $d$ is trivial on $Q_d$ or on any set with $d_{\min} = 1$ and $d_{\max} = d(n)$.
The following corollary shows that the above indeed helps to improve upon the trivial bound for appropriately bounded instances.

\begin{corollary}\label{cor:nontriv}
Lemma~\ref{sl1} implies that the greedy algorithm achieves an approximation factor of $o(d^{3/4})$ on a $(d^{3/4}, d^{3/4})$-bounded instance.
\end{corollary}
\begin{proof}
We know that $N^{\mathrm{cov}}_r(E) \ge N^{\mathrm{ent}}_r(E)$ (see for example~\cite{dz}). Therefore, Lemma~\ref{sl1} implies 
\begin{equation}\label{sqrtdcov}
    \alpha \le \left(N^{\mathrm{cov}}_{\alpha/2}(\gamma)\cdot\frac{2d-\alpha}{\alpha}\right)^{\log_2 3/2}.
\end{equation}
For $\alpha = d^{3/4}$, the right side of Equation~\ref{sqrtdcov} is $d^{log_2{3/2}}$, while the left side is $d^{3/4}$. Since $\log_2 3/2 < 3/4$ we have a contradiction. Therefore, $\alpha = o(d^{3/4}).$
\end{proof}

Of course, as the metric entropy decreases, the approximation factor improves, see for example the theorem below.

\begin{theorem}\label{main1:thm}
For $\eta = O(d^{\frac{1}{\xi\log_2 3/2}})$, ($\xi>1$), Lemma~\ref{sl1} implies that the greedy algorithm achieves an approximation factor of $\max\{2\zeta, O\left(d^{\frac{1+\xi\log_2(3/2)}{\xi(1+\log_2(3/2))}}\right)\}$ on a $(\eta,\zeta)$-bounded minimum weight matching instance. 
\end{theorem}

Together with Theorem~\ref{thm:ps}, Theorem~\ref{main1:thm} implies that the greedy algorithm obtains the approximation factor on a $(\eta,\zeta)$-bounded Wasserstein distance computation instance. Proof deferred to Section~\ref{sec:details}.

\subsection{Small random samples suffice}

In this subsection, our main contribution is Theorem~\ref{thm:randsamp}. we show that if the metric entropy is small, and so is the spread (see Definition~\ref{dfn:spread}) of the underlying distribution, then the empirical distribution of a much (polynomially) smaller sample is close to the actual distribution, in the 1-Wasserstein metric, with high probability.

\begin{definition}\label{dfn:spread}
Let $\mu$ be the uniform distribution supported on a subset of vertices $Q$ of $Q_{d(n)}$.
The {\em spread} of $\mu$, $S(\mu)$, is defined as:
\begin{equation}
    S(\mu):= \inf_{x_0\in Q}\left(1+\ln \left(\int_{Q}e^{d(x_0,x)^2}d\mu(x)\right)\right)^{1/2},
\end{equation}
where $d(\cdot,\cdot)$ denotes the $\ell_1$ distance on $Q_{d(n)}$.
\end{definition}
Note that the spread is positive and greater than $1$, for any distribution defined on the hypercube, since the minimum value of $d(\cdot,\cdot)$ is $1$. In general, $S(\mu)$ can be a function of $d(n)$.
\begin{theorem}\label{thm:randsamp}
For a $(\eta, \zeta)$ coverable point-set, with $m=\alpha(n)\left(\eta 2^\zeta\right)$ and $\alpha(n)\in(0,1)$, the 1-Wasserstein distance between the empirical distribution and the true distribution of data-sets with bounded metric entropy obeys the following Sanov type concentration bound:
\begin{equation}
%  {\scriptstyle
      \exists\ \alpha(n)\to 0,\ \lim_{n\to\infty}\frac{1}{\eta 2^\zeta}\ln \mathbb{P}\left(W_1(\hat{\mu}_m, \mu) \ge \log\log n+o(1)S(\mu)\right) \le -\Omega(1).
%  }
\end{equation}
\end{theorem}

The proof is in Section~\ref{sec:details}. It closely follows the covering based proof of multidimensional Cramer's theorem in its metric entropy version (exercise 6.2.19 in~\cite{dz}) with two main differences: (1) we need the topology induced by the Wasserstein metric instead of the Levy metric, and (2) our space has large dimension, i.e., say $\log n$, which depends upon $n$.
% \footnote{This is one reason why we can't directly apply the result from~\cite{bgv}. The constants in the exponential in their theorems will depend on $n$, and it's not immediately clear whether the dependence will lead to a non-trivial result.}. 
The second point requires us to be more careful with the covering argument, and so we only prove a relatively weaker result, with the help of the transportation inequality from~\cite{bv}.

\subsection{Greedy with random sampling}

Theorems~\ref{thm:randsamp},~\ref{main1:thm}, and~\ref{thm:ps} imply the following efficiency guarantee about the greedy minimum weight bipartite matching algorithm (GreedyMatch) on a random sample, and therefore Algorithm~\ref{algmain} as well.
\begin{theorem}\label{thm:main-eff}
  Suppose we are given two data-sets with $\mathcal{S}$ and $\mathcal{T}$ that are weighted according to distributions $\P_\mathcal{S}$ and $\P_\mathcal{T}$. If, 
\begin{enumerate}
\item $\S\cup \T$ admits a small covering: an $(\eta, \zeta)$ covering with $\eta,\zeta = O(\log^cn)$ and $\eta = O(\log^cn)$ for some constant $c \le \frac{1}{\xi(1+\log_2(3/2)})$, for any $\xi>1$; and
\item $\P_\S$ and $\P_\T$ are sufficiently far apart: $W_1(\P_\mathcal{S},\P_\mathcal{T})\ge\log\log n+o(1)\max\{S(\P_\mathcal{S}), S(\P_\mathcal{T})\}$
\end{enumerate}
then the greedy algorithm achieves an approximation ratio of $\max\left(2\zeta, O\left(d^{\frac{1+\xi\log_2(3/2)}{\xi(1+\log_2(3/2))}}\right)\right)$ with probability $1-o(1)$, when computed on a small random sample of $r(n)$ fraction of data-points and $r(n)\to 0$.
\end{theorem}

More succinctly, Theorem~\ref{thm:main-eff} states that the greedy algorithm on a small random sample can be used to approximate $W_1(\P_\mathcal{S},\P_\mathcal{T})$ on our data-sets $\mathcal{S}$ and $\mathcal{T}$, as long as the data-sets admit a small size covering using balls of small radius, and the two training weight distributions $\P_\S$ and $\P_\T$ are sufficiently different, which is the interesting case.
\section{Proofs and details}\label{sec:details}
\subsection{Proof of Theorem~\ref{wassthm}}

\begin{proof}
Let $\mathcal{L}_\T, \mathcal{L}_\S$ be the infinitesimal generators, and let $\rho_\T(w),\rho_\S(w)$ be the invariant measures, corresponding to the limiting process for the SGD for training on $\T$ and $\S$ respectively. Then, from our ergodicity assumption about the SGD, and the definition of invariant measures, we have:
\begin{eqnarray}
\mathcal{L}_\T^*\rho_\T(w) & = & 0,\\
\mathcal{L}_\S^*\rho_\S(w) & = & 0.
\end{eqnarray}

We know that $\mathcal{L}_\S$ is a perturbation of $\mathcal{L}_\T$. So, let
\begin{equation}
    \mathcal{L}_\S^* \rho_\T(w) = \varepsilon(w).
\end{equation}
Therefore,
\begin{equation}\label{eqn:inv1}
    \mathcal{L}_\S^* \rho_\S(w) - \mathcal{L}_\S^* \rho_\T(w) = \varepsilon(w),
\end{equation}
and
\begin{equation}\label{eqn:inv2}
    \mathcal{L}_\T^*\rho_T(w) - \mathcal{L}_\S^* \rho_\T(w) = \varepsilon(w).
\end{equation}
Putting Equations~\ref{eqn:inv1} and~\ref{eqn:inv2} together, we have:
\begin{eqnarray}\label{eqn:inv3}
    \mathcal{L}_\S^*(\rho_\S(w) - \rho_\T(w)) &=&  (\mathcal{L}_\T^* - \mathcal{L}_\S^*) \rho_\T(w)\nonumber\\
    (\rho_\S(w) - \rho_\T(w)) &=&  (\mathcal{L}_\S^*)^{-1}(\mathcal{L}_\T^* - \mathcal{L}_\S^*) \rho_\T(w),
\end{eqnarray}
where we have used the uniform ellipticity assumption in the last step to ensure the inverse exists. Taking $1$-norm on both sides and using the sub-additivity of operator norms, we have:
\begin{equation}
     \|\rho_\S(w) - \rho_\T(w)\|_1 \le  \|(\mathcal{L}_\S^*)^{-1}\|_1\|(\mathcal{L}_\T^* - \mathcal{L}_\S^*)\|_1.
\end{equation}
We will upper-bound $\|(\mathcal{L}_\T^* - \mathcal{L}_\S^*)\|_1$ in terms of  $W_1(\P_\T,\P_\S)$, but do note that the $O(\cdot)$ notation only characterizes the linear dependence on $W_1$, and there is a large constant factor hidden in the notation due to the maximum eigenvalue of the inverse of the adjoint operator above.

The essential idea is to simply write down the adjoints of the elliptic operators, group like terms together and use Kantorovich-Rubenstein duality to upper-bound each of the resulting terms in terms of $W_1(\P_\T,\P_\S)$. Recall that,
\begin{eqnarray}
\mathcal{L}_\T\psi&\equiv& \nabla_w\ell_w(\P_\T)\frac{\partial\psi}{\partial w_i} + D\frac{\partial \psi}{\partial w_i\partial w_j}, \\
\mathcal{L}_\T^*\psi&\equiv& \frac{\partial \nabla_w\ell_w(\P_\T)}{\partial w_i}\psi + \nabla_w\ell_w(\P_\T)\frac{\partial \psi}{\partial w_i} -  D\frac{\partial \psi}{\partial w_i\partial w_j},
\end{eqnarray}
where $D=\sigma_\T\sigma_\T^T$, and we have used the Einstein summation notation on the partial derivatives for the sake of brevity in expressing the last two equations. Similarly, we can write out $\mathcal{L}_\S$ and $\mathcal{L}_\S^*$.

Note that:
\begin{equation}\label{eqn:iso-adj}
    (\mathcal{L}_\T^*-\mathcal{L}_\S^*)\psi\equiv  \left(\frac{\partial \nabla_w\ell_w(\P_\T)}{\partial w_i} - \frac{\partial \nabla_w\ell_w(\P_\T)}{\partial w_i}\right)\psi + (\nabla_w\ell_w(\P_\S)-\nabla_w\ell_w(\P_\S))\frac{\partial \psi}{\partial w_i},
\end{equation}
where we have used Assumption~\ref{sgd:assm} to cancel out the second order derivative terms.

One can choose $\psi$ as any Lipschitz function of unit $\ell_1$ norm, so that 
if we upper bound the coefficients of the two partial terms on the RHS of Equation~\ref{eqn:iso-adj} for every co-ordinate $i$ by $W_1(\P_\T,\P_\S)$, then we will have bounded $\|\mathcal{L}_\T^*-\mathcal{L}_\S^*\|_1$ by $W_1(\P_\T,\P_\S)$.
The first term can be upper-bounded as:
\begin{eqnarray}\label{eqn:covcond-use}
     \nabla_w\ell_w(\P_\T)-\nabla_w\ell_w(\P_\S)
     &=&\nabla_w\E_{x\sim\P_\S}\E_{y\sim\P_\S(\cdot|x)}[(y-f(w,x))^2] - \nabla_w\E_{x\sim\P_\T}\E_{y\sim\P_\T(\cdot|x)}[(y-f(w,x))^2]\nonumber\\
     &\simeq&\nabla_w\E_{x\sim\P_\S}\E_{y\sim\P_\S(\cdot|x)}[(y-f(w,x))^2] - \nabla_w\E_{x\sim\P_\T}\E_{y\sim\P_\S(\cdot|x)}[(y-f(w,x))^2]\nonumber\\
     &\le& O(W_{1}(\P_\T,\P_\S)),
\end{eqnarray}
where we have used the Kantorovich-Rubenstein duality together with the assumption that $\E_{y\sim\P(\cdot|x)}[(y-f(w,x))^2]$ is O(1)-Lipschitz in deriving the last inequality.

Similarly, one can show the same upper-bound for $\left(\frac{\partial \nabla_w\ell_w(\P_\T)}{\partial w_i} - \frac{\partial \nabla_w\ell_w(\P_\T)}{\partial w_i}\right)$.
Therefore, $\|\mathcal{L}_\T^*-\mathcal{L}_\S^*\|_1\le O(1) W_1(\P_\T,\P_\S)$. 
\end{proof}
\begin{corollary}\label{cor:wassthm}
The anisotropic diffusivity case: The upper-bound holds when the diffusivity is anisotropic as well, with the caveat that $W_1(\P_\mathcal{S},\P_\mathcal{T})$ be replaced by $W_1(\P_\mathcal{S},\P_\mathcal{T})^2$.
\end{corollary}
\begin{proof}
 The proof of Theorem~\ref{wassthm} uses isotropic diffusivity in one place only -- when computing the difference $\mathcal{L}^*_\mathcal{T} - \mathcal{L}^*_\mathcal{S}$. Note that, the diffusivity may be written as (see for example~\cite{diff-pc}):
\begin{equation}
D(\P):=\E\left[\nabla\ell_w(\P)\cdot\nabla\ell_w(\P)^T\right] - \E\left[\nabla\ell_w(\P)\right]\cdot\E\left[\nabla\ell_w(\P)^T\right].
\end{equation}
Then, $D(\P_\mathcal{S})-D(\P_\mathcal{T})$ can be upper-bounded in terms of $O(W_1(\P_\mathcal{S},\P_\mathcal{T}))+O(W_1(\P_\mathcal{S},\P_\mathcal{T})^2)$ under the assumption that we have $O(1)$-Lipschitz gradients. The argument is similar to that used for the drift term in the isotropic case, albeit with one new observation, the term 
\[\E_{\P_\mathcal{S}\times\P_\mathcal{S}}\left[f(w)\right] - \E_{\P_\mathcal{T}\times\P_\mathcal{T}}\left[f(w)\right]\]
can be upper-bounded by $W_1(\P_\mathcal{S},\P_\mathcal{T})^2$ using Kantorovich-Rubenstein duality and the definition of Wasserstein distance.
\end{proof}

\subsection{Further details about the greedy bipartite matching algorithm}
For the sake of completeness, and its relevance to the proof of Lemma~\ref{sl1} (presented next), we outline below a way to implement the greedy bipartite matching algorithm as Algorithm~\ref{alg1}.
\begin{algorithm}[htb]
\caption{GreedyMatch (BFS based greedy bipartite matching)}\label{alg1}
\begin{algorithmic}[1]
\STATE{\bf GreedyMatch}($R,B$)
\STATE {\bf Input:} Two multi-sets of $n$ points $R,B$ in $Q_d$.
\STATE {\bf Output:} A matching from $R$ to $B$.
%\Statex $\triangleright$ The set B is shared across all threads
\STATE For {$r \in R$}\qquad $\triangleright$ All for loop statements run in parallel
\STATE \quad$b\gets\mathrm{BreadthFirstSearch}(r,B)$\quad $\triangleright$ Find the vertex in $B$ closest to $r$ and match it to $r$ (break ties arbitrarily)
\STATE \quad$M\gets M\cup \{r\to b\}$
%\EndFor
\STATE \textbf{return} $M$
\STATE
\STATE {\bf BreadthFirstSearch}($r,B$)
\STATE For {$i=1,...,d$}
\STATE \quad For {$v\in Q_d$, $\|v-r\|_1=i$}
\STATE \quad If {$v\in B$}
\STATE \qquad$B\gets B\setminus{v}$
\STATE \textbf{return} $v$\qquad $\triangleright$ $r$ matches to $v$
% \EndIf
% \EndFor
% \EndFor
\end{algorithmic}
\end{algorithm}

\subsection{Proof of Lemma~\ref{sl1}}

\begin{proof}
One way to write the greedy matching algorithm is to imagine it as a set of parallel breadth first searches (BFS) (see for e.g.~\ref{alg1}). For any two vertices $x,y\in\mathrm{R}$ in an alternating cycle $\gamma$, suppose their neighbors from the greedy matching algorithm are $x'$ and $y'$ respectively. Think of the step in the BFS before either $x$ or $y$ was matched, so at that time-point the BFS from $x'$ and $y'$ hadn't reached either $x$ or $y$. Therefore, we have the following relationship between their mutual distances:
\begin{equation}\label{greedyprop}
    \min\{d(x',y), d(y',x)\} \ge \min\{d(x,x'), d(y,y') \},
\end{equation}
where $d$ denotes the distance metric, which in our case is the underlying cost in $W_1$ i.e, the $\ell_1$ distance.
The situation is illustrated in Figure~\ref{fig1}.
\begin{figure}[h!]
    \centering
    \includegraphics[width=10.5cm, height=5.5cm]{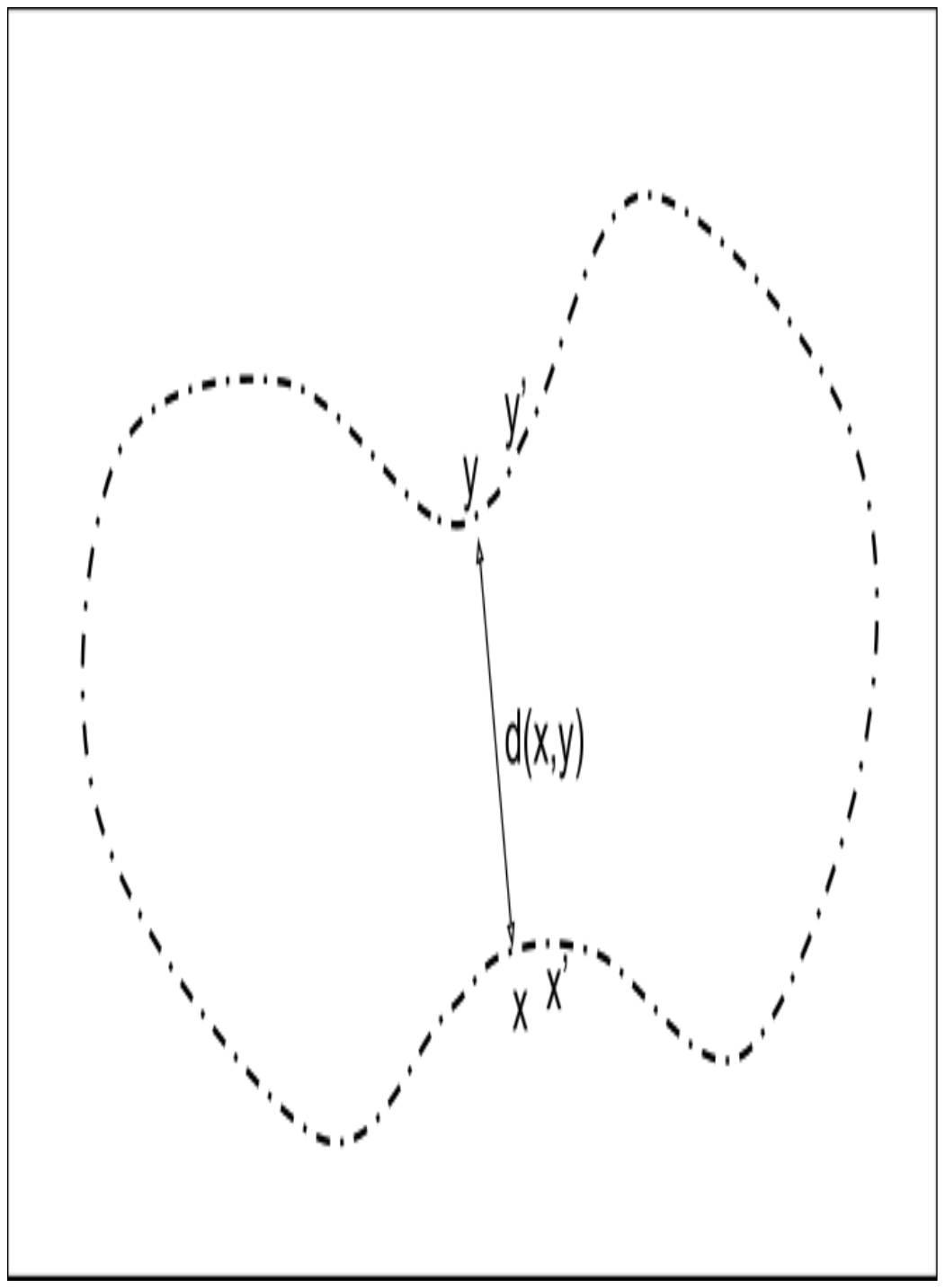}
    \caption{Alternate edges in an alternating cycle $\gamma$ belong to greedy and optimal matching.}
    \label{fig1}
\end{figure}

Now suppose that the weight of the greedy matching edges in $\gamma$ is $\alpha$ times the weight of the minimum weight matching. Then we show below that a significant fraction of the greedy edges in the cycle $\gamma$ must be at a distance at least $\alpha/2$ from their neighbors. 

Let $G$ be the set of greedy edges in $\gamma$ and $M$ be the set of optimal matching edges. Then we have,
\begin{equation}
    \sum_{xy\in G} d(x,y) \ge \alpha\sum_{xy\in M} d(x,y).
\end{equation}
Let $f$ be the fraction of edges in $G$ with weight at least $\alpha/2$. Let's call that set $G_{\alpha/2}$. Recall that, in the setting of $Q_d$, $d_{\min} = 1$ and $d_{\max} = d$. Therefore, we have
\begin{equation}
    d\cdot f + \frac{\alpha}{2}\cdot(1-f)\ge \alpha.
\end{equation}
Therefore, $f\ge\frac{\alpha}{2d-\alpha}$.

By the definition of metric entropy and Equation~\ref{greedyprop}, we know that
\begin{equation}\label{mentprop}
    |\gamma|f\le N^{\mathrm{ent}}_{\alpha/2}(G_{\alpha/2})\le N^{\mathrm{ent}}_{\alpha/2}(\gamma).
\end{equation} 
By Theorem~\ref{tjn} we know that $\alpha\le|\gamma|^{\log_2 3/2}$. Putting that together with Equation~\ref{mentprop} gives:
\begin{equation}
    \alpha \le \left(N^{\mathrm{ent}}_{\alpha/2}(\gamma)\cdot\frac{2d-\alpha}{\alpha}\right)^{\log_2 3/2}.
\end{equation}
\end{proof}

\subsection{Proof of Theorem~\ref{main1:thm}}
\begin{proof}
We have two cases:
\begin{enumerate}
    \item $\alpha \le 2\zeta$: In this case, there's nothing to prove.
    \item $\alpha \ge 2\zeta$: In this case, since $N^{\mathrm{cov}}_{\alpha/2}(\gamma)\ge N^{\mathrm{cov}}_{\zeta}(\gamma) = \eta$, we have 
    \begin{equation}\label{age2e}
    \alpha \le \left(\eta\cdot\frac{2d-\alpha}{\alpha}\right)^{\log_2 3/2}.
    \end{equation}
    Therefore, we have two sub-cases:
    \begin{enumerate}
    \item $\alpha = \Omega(d)$: In this case, we obtain from Equation~\ref{age2e} that $\alpha = O(\eta^{\log_2 3/2})$, which is $o(d)$ for $\eta = O(d^{\frac{1}{\xi\log_2 3/2}})$ --  a contradiction for $\xi>1$. Hence $\alpha = o(d)$.
    \item $\alpha = o(d)$: In this case we obtain:
    \begin{eqnarray}
    \alpha^{1+\log_2{3/2}} &\le& \left(\eta\cdot 2d\right)^{\log_2 3/2}\nonumber\\
    \alpha &\le& O\left(d^{\frac{1+\xi\log_2(3/2)}{\xi(1+\log_2(3/2))}}\right),\label{age2e2}
    \end{eqnarray}
    where we have used $\eta = O(d^{\frac{1}{\xi\log_2 3/2}})$ in the last inequality.
    \end{enumerate}
\end{enumerate}
\end{proof}

\subsection{Proof of Theorem~\ref{thm:randsamp}}

\begin{proof}
Recall that, we have an $(\eta,\zeta) = (\log^{c_1}n,\log^{c_2}n)$ instance, for some small constants $c_1$ and $c_2$. So the covering number of the support set for $\mu$, denoted $\mathcal{S}_\mu$, with balls of radius $\delta$ ($\delta\in[1,\zeta)$), denoted $m(\mathcal{S}_\mu,\delta)$, is upper bounded by $\eta\cdot\frac{\mathrm{Vol}(\zeta,Q_{d(n)})}{\mathrm{Vol}(\delta, Q_{d(n)})}$. Replacing the asymptotic value for the volume, we get
\begin{equation}\label{eqn:ent-est}
    \eta\cdot\frac{\mathrm{Vol}(\zeta,Q_{d(n)})}{\mathrm{Vol}(\delta, Q_{d(n)})}\le \eta\cdot 2^{-d(n)(H(\zeta/d(n))-H(\delta/d(n)))}-o(1),
\end{equation}
where $H(x):=x\log_2 x + (1-x)\log_2(1-x)$ is the entropy function and is negative for $x\in(0,1)$.

Since $\mathcal{S}_\mu$ of $\mu$ is finite, the set of set of probability measures $M_1$ that are supported on $\mathcal{S}_\mu$ is compact in the 1-Wasserstein metric topology. Therefore, there exists a finite covering of $M_1$, i.e., using elementary measures that are constant on the atoms of a finite covering of $\mathcal{S}_\mu$, we can approximate any given probability measure in $M_1$ up to an additive constant $\varepsilon+\delta$, in the 1-Wasserstein metric. The value of the constant for each ball in the covering ranging in $[0,1]$ in steps of $\varepsilon$. We will fix the values of $\varepsilon\in(0,1)$ and $\delta\in[1,\zeta)$ later in the proof.

Therefore, as in exercise 6.2.19 in~\cite{dz}, we can bound the covering number of $M_1$, i.e., $m(M_1,\delta,\varepsilon)$ by 
\begin{equation}
    m\left(M_1,\delta,\varepsilon\right)\le {m(\mathcal{S}_\mu,\delta)(1 + \frac{1}{\varepsilon}) \choose m(\mathcal{S}_\mu, \delta)}\le \left(\frac{4}{\varepsilon}\right)^{m(\mathcal{S}_\mu, \delta)}.
\end{equation}

Therefore, we have by the standard covering argument for the proof of multidimensional version of Cramer's large deviation bound (equivalently Sanov's theorem for finite spaces, see exercise 6.2.19 in~\cite{dz}):
\begin{equation}\label{eqn:ld-rs}
    \exists m_0\ \forall m>m_0,\ \mathbb{P}(\hat{\mu}_m\in A)\le m\left(M_1,\delta,\varepsilon\right)\cdot e^{-m\cdot\inf_{\nu\in A^{\varepsilon+\delta}}H(\nu,\mu)},
\end{equation}
where $H(\nu,\mu)$ is the relative entropy (KL divergence), and $A^\delta$ is the $\delta$ blow-up of $A\subset M_1$ with respect to the 1-Wasserstein metric.

Note that $\inf_{\nu\in A^{\varepsilon+\delta}}H(\nu,\mu)$ can be lower bounded in terms of the 1-Wasserstein distance using the following transportation inequality from~\cite{bv}.
\begin{theorem}\cite{bv}
  For distribution $\mu,\nu$ supported on any polish space, we have:
  \begin{equation}
      H(\mu,\nu) S(\mu)\ge W_1(\mu,\nu).
  \end{equation}
\end{theorem}
Essentially, 
\begin{equation}
    \inf_{\nu\in A^{\varepsilon+\delta}}H(\nu,\mu)\ge \inf_{\nu\in A^{\varepsilon+\delta}}\frac{W(\nu,\mu)}{S(\mu)}\ge \frac{W(\nu,\mu)- \delta - \varepsilon}{S(\mu)}.
\end{equation}
For $\varepsilon \ll 1$, we have $\delta+\varepsilon\simeq\delta$. Therefore, the exponent on the RHS of Equation~\ref{eqn:ld-rs} can be lower bounded as
\begin{equation}
    \inf_{\nu\in A^{\varepsilon+\delta}}H(\nu,\mu)\ge \frac{W(\nu,\mu)-\delta}{S(\mu)} .
\end{equation}
Furthermore, for $m=\alpha(n)|Q|=\alpha(n)\left(\eta 2^\zeta\right)$, the RHS of Equation~\ref{eqn:ld-rs} can be upper bounded as 
\begin{eqnarray}\label{eqn:was-san}
\mathbb{P}(\hat{\mu}_m\in A)&\le&\left(\frac{4}{\varepsilon}\right)^{m(\mathcal{S}_\mu, \delta)}\cdot e^{-\left(\eta 2^{-d(n)H(\zeta/d(n))}\right)\alpha(n)\cdot\left(\frac{W(\nu,\mu)- \delta}{S(\mu)}\right)}\nonumber\\
&\le&e^{\eta\cdot 2^{-d(n)(H(\zeta/d(n))-H(\delta/d(n)))}\ln\left(\frac{4}{\varepsilon}\right)}\cdot e^{-\left(\eta 2^{-d(n)H(\zeta/d(n))}\right)\alpha(n)\cdot\left(\frac{W(\nu,\mu)- \delta}{S(\mu)}\right)},
\end{eqnarray}
where we have used Equation~\ref{eqn:ent-est} in the last inequality. If we choose $\delta$ and $\varepsilon$ such that
\begin{eqnarray}
    \left(\frac{W(\nu,\mu)- \delta}{S(\mu)}\right)&\ge& \frac{2^{d(n)H(\delta/d(n))}\ln\left(\frac{4}{\varepsilon}\right)}{\alpha(n)}\nonumber\\
    W(\nu,\mu) &\ge& \delta + \frac{2^{d(n)H(\delta/d(n))}\ln\left(\frac{4}{\varepsilon}\right)}{\alpha(n)}S(\mu),
\end{eqnarray}
equivalently for a small enough $\alpha(n)$, say $\alpha(n) = \log\log n$, choose $\delta(n) = \log\log n$ and $\varepsilon > 0$ then $W(\nu,\mu)\ge\log\log n +o(1)S(\mu)$ and the exponent in Equation~\ref{eqn:was-san} is negative. Thus $\mathbb{P}(\hat{\mu}_n\in A)\to 0$ as $n\to\infty$.
\end{proof}

\section{Experimental results}\label{sec:expt}
MNK1 and MNK2 are two structurally similar kinases responsible for cell signalling. Their inhibition has been explored for certain cancer therapies (see for example, the survey~\cite{mnk-survey}). In this section, we describe an application of Algorithm~\ref{algmain} that selects for MNK2 binders which are MNK1 non-binders. 

We use (roughly) the same set-up as in~\cite{pfr} -- a graph convolutional neural network with cross entropy loss.\footnote{Our proofs are for mean square error loss but the similarity between the two loss functions means that the general ideas should continue to hold.} At a high level, we start with a train data-set $\mathcal{S}$, labeled as hits (binders) and non-hits (non-binders) for MNK2, and another labeled set $\mathcal{T}$ which consists of just MNK1 non-hits (non-binders)\footnote{
These labeled training data-sets are proprietary and can't be disclosed, as was the case in~\cite{pfr}. Hence we train on proprietary data-sets, but perform inference on publicly available data-sets (Enamine and MCULE catalogs). to validate the model prospectively, we experimentally tested our top predicted selective molecules from the publicly available Enamine catalog, which results in 2 molecules (of 43 tested) verified to be selective by single point of concentration assays at 10$\mu$M. See Subsection~\ref{wet-lab:sn} for more details.
}. Let $\P_\mathcal{S}$ and $\P_\mathcal{T}$ denote the distribution of weights on the examples in the two data-sets. We assume $\P_\mathcal{S}$ and $\P_\mathcal{T}$ are uniform in our experiment section, but their set of support is different. Next, we used $\alpha = 0.95$ in Algorithm~\ref{algmain} to reweigh the MNK2 binders in $\mathcal{S}$ for training using $\mathcal{T}$ (the MNK1 non-binders). In theory, this should bring the limiting distribution of network weights, that results from training on reweighed $\P_\mathcal{S}$, "closer" to that which can be obtained from training on $\P_\mathcal{T}$. That's our experimental model. For the baseline model, we simply skip the reweighing step in Algorithm~\ref{algmain} and train the network using SGD.

\subsection{Further details about the experimental setup}\label{sbs:expt-det}
Below we provide some more details of our experiments beyond the high level description in the main paper. We use a scaled down but otherwise same setup as in~\cite{pfr}. The experiment consists of two sets of training data:
\begin{enumerate}
    \item Baseline data: The baseline training set consists of labeled disynthon examples divided into five classes:
    (1) Non-hit, (2)  Matrix binder, (3) Promiscous hit, (4) Non-competitive hit (5) Competitive hit. Of these five, the class of interest is competitive hit and has about 1,200 molecules. The overall  training set size is approximately $250$K labeled disynthons (similar to, but a scaled down version of~\cite{pfr}) for MNK1 and MNK2  combined.\footnote{Because of our reliance on proprietary disynthon libraries, the training data-sets can't be open sourced, as was the case in~\cite{pfr}.}
    \item Experiment data: This is exactly the same as baseline data, except for one caveat: molecules that are competitive binders to MNK2 and close to non-hits to MNK1 in the molecular fingerprint space, have their weights increased in the loss function, using the transportation algorithm (Algorithm~\ref{algmain}) described in the previous section. In other words, small molecules close to non-hits to MNK1 are weighted relatively higher amongst the competitive binders to MNK2.
   \item We use a holdout of our labeled data-set, that consists of about $7$K small molecules that belong to the labeled set of: MNK2 binders, and MNK1 binders or MNK1 non-binders; to compute the selectivity of our algorithm. About $40\%$ of this set was labeled MNK2 binders and MNK1 non-binders. For both models we obtain a set of top 100 molecules in the holdout data that are predicted to be the strongest binders to MNK2. 
\end{enumerate}

\subsection{Assay results}\label{wet-lab:sn}

Due to proprietary reasons we can not make our code and data public. However, to ensure a degree of verifiability of our results, about fifty of the top predicted selective small molecules from the enamine catalog were synthesized and experimentally verified for binding properties to MNK1 and MNK2. This subsection discusses those results. 

We selected 50 of the top predicted small molecules that should bind to MNK2 but not to MNK1 from the enamine catalog of 1.9 Billion small molecules, for synthesis and experimental testing. 
The selection process essentially filtered out any small molecule that was above a ECFP6 Tanimoto similarity of $0.3$ with respect to an already selected (higher score) small molecule in the top predicted selective molecules. At $10\mu M$ concentration, it was found that 2 out of 43\footnote{Seven molecules could not be synthesized and tested.}  (roughly $5\%$) of predicted small molecules reduced the enzyme activity of MNK2 below 50\% but not that of MNK1. 

\begin{remark}
Note that most small molecules that bind to MNK2 will also bind to MNK1 because of structural similarities. Admittedly, our training set is an order of magnitude smaller than that in~\cite{pfr} and the number of molecules experimentally verified is also an order of magnitude smaller than~\cite{pfr}. However, if the result scales up, then the $5\%$ experimentally verified success rate, in predicting selectivity against two targets simultaneously, may be further compared with the $30\%$ success rate for predicting binders against single targets using similar ML approaches~\cite{pfr}.
\end{remark}

The names of the two molecules in Figure~\ref{fig:mcul-wetlab} are:
\begin{enumerate}
    \item For C(NC=1N=CC=C(OC=2C=CC=CC2)N1)C=3C=CC(=CC3)C=4C=CC=NC4, the IUPAC name is:\\ 4-phenoxy-N-{[4-(pyridin-3-yl)phenyl]methyl}pyrimidin-2-amine,
    \item For C(C1CCN(CC1)C=2C=NC=C(OC3CCCCC3)N2)C=4C=CC=NC4, the IUPAC name is:\\ 2-(cyclohexyloxy)-6-{4-[(pyridin-3-yl)methyl]piperidin-1-yl}pyrazine.
\end{enumerate}

On average the model predicts molecules which are better binders to MNK2 than MNK1, as can be seen by the plot of the enzyme activity of the forty three  small molecules below.

\makeatletter
\setlength{\@fptop}{0pt}
\makeatother
\begin{figure}[ht!]
\centering
    \includegraphics[scale=0.12]{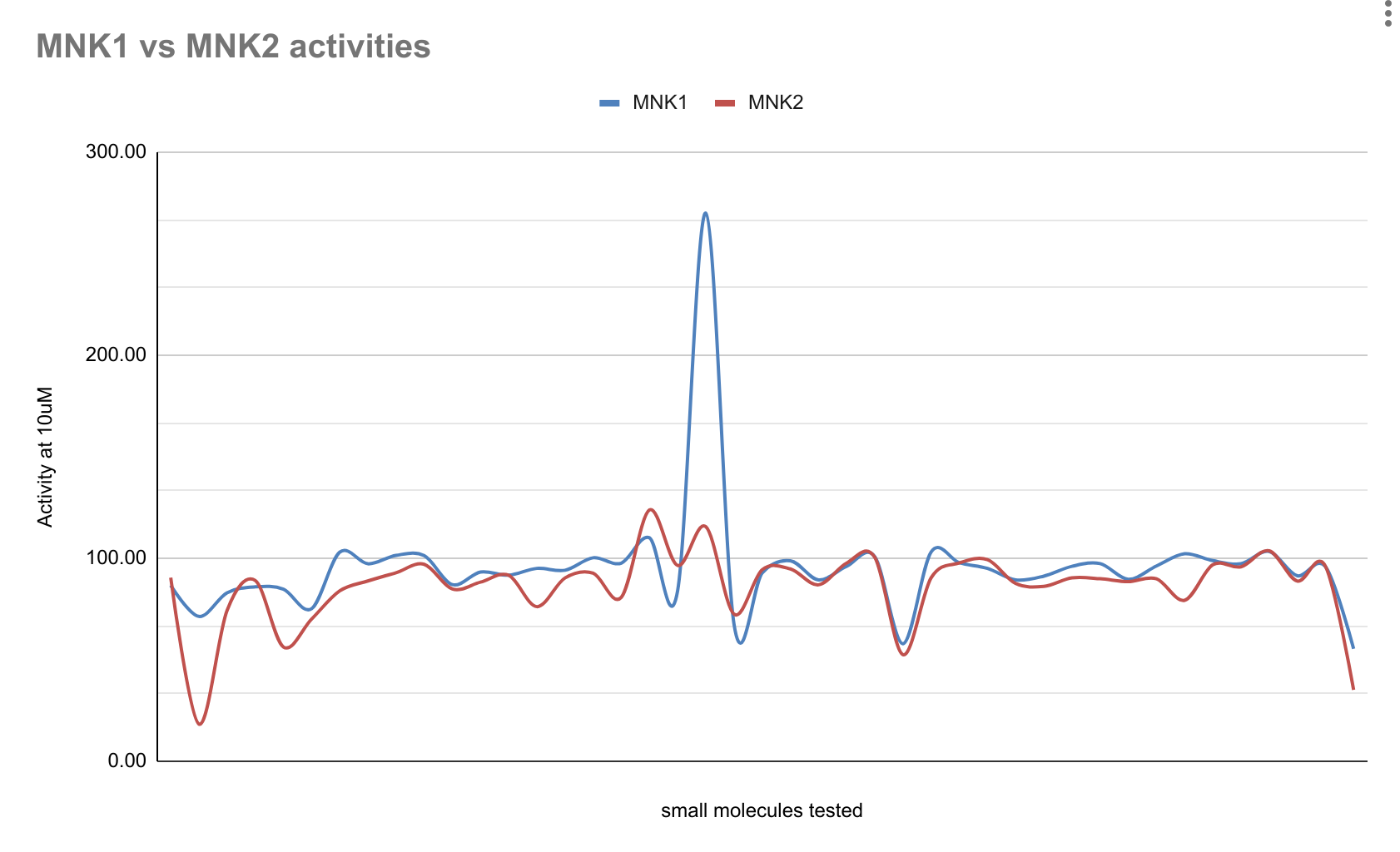}
\caption{Average enzyme activities; lower enzyme activities ($y$-values) indicate better binders}
\label{fig:wetlab-avg}
\end{figure}
\end{document}